\documentclass[10pt]{article}
\usepackage[preprint]{tmlr}

\usepackage{amsmath,amssymb,amsthm}
\usepackage{booktabs}
\usepackage{multirow}
\usepackage{graphicx}
\usepackage{xcolor}
\usepackage{pifont}
\usepackage{subcaption}
\usepackage{tikz}
\usetikzlibrary{arrows.meta,positioning,calc,fit,backgrounds,shapes.geometric,decorations.pathreplacing}
\usepackage{algorithm}
\usepackage{algorithmic}
\usepackage{listings}
\usepackage{enumitem}
\usepackage{mathtools}
\usepackage{bm}
\usepackage{microtype}
\usepackage{xspace}
\usepackage{hyperref}
\usepackage{url}

\graphicspath{{figures/}}

\newtheorem{theorem}{Theorem}[section]
\newtheorem{proposition}[theorem]{Proposition}
\newtheorem{lemma}[theorem]{Lemma}
\newtheorem{corollary}[theorem]{Corollary}
\theoremstyle{definition}
\newtheorem{definition}[theorem]{Definition}
\theoremstyle{remark}
\newtheorem{remark}[theorem]{Remark}

\newcommand{\edelta}{E$\Delta$-MHC-Geo\xspace}

\newcommand{\R}{\mathbb{R}}
\newcommand{\I}{\mathbf{I}}
\newcommand{\bx}{\mathbf{x}}
\newcommand{\bX}{\mathbf{X}}
\newcommand{\bQ}{\mathbf{Q}}
\newcommand{\bH}{\mathbf{H}}
\newcommand{\bA}{\mathbf{A}}
\newcommand{\bM}{\mathbf{M}}
\newcommand{\bu}{\mathbf{u}}
\newcommand{\bv}{\mathbf{v}}
\newcommand{\bk}{\mathbf{k}}
\newcommand{\bW}{\mathbf{W}}
\newcommand{\by}{\mathbf{y}}
\newcommand{\bb}{\mathbf{b}}
\newcommand{\be}{\mathbf{e}}
\newcommand{\SO}{\mathrm{SO}}
\newcommand{\On}{\mathrm{O}}

\tikzset{
  block/.style={rectangle, draw, rounded corners=2pt, minimum height=1.8em, minimum width=3em, align=center, font=\small},
  greenblock/.style={block, fill=green!10, draw=green!60!black},
  redblock/.style={block, fill=red!8, draw=red!60!black},
  blueblock/.style={block, fill=blue!8, draw=blue!60!black},
  orangeblock/.style={block, fill=orange!10, draw=orange!60!black},
  grayblock/.style={block, fill=gray!8, draw=gray!60},
  circleop/.style={circle, draw, inner sep=1.5pt, font=\scriptsize\bfseries},
  arrowstyle/.style={-{Stealth[length=5pt]}, thick},
  dasharrow/.style={-{Stealth[length=5pt]}, thick, dashed},
}

\def\month{02}
\def\year{2026}
\def\openreview{\url{https://openreview.net/forum?id=XXXX}}

\title{\texorpdfstring{The E$\Delta$-MHC-Geo Transformer: Adaptive Geodesic Operations with Guaranteed Orthogonality}{The EDelta-MHC-Geo Transformer: Adaptive Geodesic Operations with Guaranteed Orthogonality}}

\author{
  \name Arash Shahmansoori \email arash.mansoori65@gmail.com \\
  \addr Independent Researcher
}

\begin{document}
\maketitle

\begin{abstract}
We present the \edelta Transformer, a novel architecture that unifies
Manifold-Constrained Hyper-Connections (mHC), Deep Delta Learning (DDL), and
the Cayley transform to obtain \emph{input-adaptive, unconditionally orthogonal}
residual connections.  Unlike DDL, whose Householder operator is orthogonal only
at $\beta\!\in\!\{0,2\}$, our Data-Dependent Cayley rotation
$\bQ(\bx)\!=\!(\I+\tfrac{\beta}{2}\bA(\bx))^{-1}(\I-\tfrac{\beta}{2}\bA(\bx))$
preserves orthogonality for \emph{all} $\beta$ and all inputs.
To handle negation (eigenvalue $-1$, which Cayley provably excludes), we
introduce the \edelta Hybrid that combines Cayley rotation with Householder
reflection via a learned operator-selection gate
$\bX'\!=\!\gamma(\bX)\,\bQ(\bX)\bX + (1-\gamma(\bX))\,\bH_2(\bX)\bX$.
A midpoint-collapse regularizer $\mathcal{L}_{\mathrm{gate}}=4\gamma(1-\gamma)$
encourages boundary gate decisions, where each selected component is
orthogonal.
In matched-parameter comparisons (${\sim}1.79$M parameters each, mean $\pm$
std over 3 seeds) against four baselines including the concurrent
JPmHC~\citep{jpmhc2026}, \edelta achieves the best long-horizon stability
($1.9\times$ over JPmHC, $3.8\times$ over GPT), the best near-$\pi$
rotation loss ($4.5\times$ over JPmHC on single-plane), strong norm
preservation ($0.001$ mean deviation), and $0.96$ negation cosine
alignment in a diagnostic reflection probe---all with $33\%$ fewer layers.
While JPmHC's wider representation ($n_{\mathrm{embd}}\!=\!512$) excels
on pure rotation (gyroscope), its finite Cayley residual mixer excludes an
exact $\lambda\!=\!{-1}$ operator and has no reflection branch, motivating
our hybrid approach for accessing both connected components of $\On(n)$.
Code is available at \url{https://github.com/arash-shahmansoori/edelta}.
\end{abstract}

\section{Introduction}\label{sec:intro}

Residual connections~\citep{he2016deep} are a cornerstone of modern deep
learning, enabling gradient flow through very deep networks via the shortcut
$\bX_{l+1}=\bX_l+F(\bX_l)$.  The standard additive residual, however,
provides no geometric guarantees: norms can drift, and the identity shortcut
limits the expressivity of inter-layer mixing.

Two recent lines of work address this limitation from complementary angles.
\textbf{Manifold-Constrained Hyper-Connections
(mHC)}~\citep{deepseek2024mhc} replace the identity with a multi-stream
residual mixing matrix $\bH_{\mathrm{res}}$ (projected to doubly stochastic
via Sinkhorn--Knopp), together with pre/post mappings
$\bH_{\mathrm{pre}},\bH_{\mathrm{post}}$ for stream aggregation.
\textbf{Deep Delta Learning (DDL)}~\citep{ddl2024} replaces the identity
shortcut with a Householder operator
$\bH_\beta=\I-\beta\bk\bk^\top$ whose direction $\bk$ and magnitude $\beta$
are input-dependent, yielding an \emph{input-adaptive} residual.

\textbf{Limitations.} DDL is orthogonal \emph{only} when
$\beta\in\{0,2\}$; during training $\beta$ varies continuously, breaking
isometry and causing gradient instability.  mHC's Sinkhorn projection is
only approximately orthogonal, and errors accumulate over long sequences.
Neither architecture provides \emph{unconditional} orthogonality.

\textbf{Our contribution.}  We propose \edelta, which replaces the residual
mixing matrix with a \emph{Data-Dependent Cayley rotation}
$\bQ(\bx)\in\SO(n)$.  The key insight is that the skew-symmetry of the
Cayley generator $\bA=\bu\bv^\top-\bv\bu^\top$ depends only on its
algebraic form, not on how $\bu,\bv$ are obtained.  Making
$\bu(\bx),\bv(\bx)$ input-dependent therefore preserves \emph{all} Cayley
properties---orthogonality, isometry, $\det=+1$---for every input and every
$\beta$.

Since Cayley provably cannot produce eigenvalue $-1$ (Theorem~\ref{thm:exclusion}),
we introduce the \edelta \textbf{Hybrid}, combining Cayley rotation with
Householder reflection ($\beta=2$ fixed) through a learned gate~$\gamma$.
A midpoint-collapse regularizer encourages $\gamma\to\{0,1\}$, so the
model can ``jump'' between the two disconnected components of $\On(n)$
rather than linger in the non-orthogonal interior.

\begin{figure}[t]
\centering
\begin{tikzpicture}[
  node distance=0.6cm and 0.8cm,
  every node/.style={font=\footnotesize},
  >=Stealth,
  scale=0.62, transform shape,
]
\begin{scope}[local bounding box=panelA]
  \node[grayblock, minimum width=4em] (xa) {$\bX_l$};
  \node[blueblock, below=0.6cm of xa] (fa) {$F(\cdot)$};
  \node[circleop, below=0.6cm of fa] (plusa) {$+$};
  \node[grayblock, below=0.6cm of plusa, minimum width=5.5em] (outa) {$\bX_l\!+\!F(\bX_l)$};
  \draw[arrowstyle] (xa) -- (fa);
  \draw[arrowstyle] (fa) -- (plusa);
  \draw[arrowstyle] (xa.east) -- ++(0.7,0) |- (plusa.east);
  \draw[arrowstyle] (plusa) -- (outa);
  \node[above=0.15cm of xa, font=\small\bfseries] {(a) Standard Residual};
  \node[below=0.1cm of outa, font=\scriptsize, text=gray] {$\det\!=\!+1$, no orthogonality};
\end{scope}

\begin{scope}[local bounding box=panelB, xshift=6.0cm]
  \node[grayblock, minimum width=4em] (xb) {$\bX_l$};
  \node[orangeblock, below=0.6cm of xb] (hb) {$\bH_\beta\!=\!\I\!-\!\beta\bk\bk^\top$};
  \node[orangeblock, right=0.4cm of hb] (rb) {$\beta\bk\bv^\top$};
  \node[circleop, below=0.6cm of hb] (plusb) {$+$};
  \node[grayblock, below=0.6cm of plusb, minimum width=6em] (outb) {$\bH_\beta\bX\!+\!\beta\bk\bv^\top$};
  \draw[arrowstyle] (xb) -- (hb);
  \draw[arrowstyle] (xb.east) -- ++(0.45,0) |- ([yshift=0.25cm]rb.north) -- (rb.north);
  \draw[arrowstyle] (hb) -- (plusb);
  \draw[arrowstyle] (rb.south) |- (plusb.east);
  \draw[arrowstyle] (plusb) -- (outb);
  \node[above=0.15cm of xb, font=\small\bfseries] {(b) DDL};
  \node[below=0.1cm of outb, font=\scriptsize, text=gray] {$\det\!=\!{-1}$, orth.\ iff $\beta\!\in\!\{0,2\}$};
\end{scope}

\begin{scope}[local bounding box=panelJP, xshift=12.6cm]
  \node[grayblock, minimum width=4em] (xj) {$\bX_l$};
  \node[blueblock, below left=0.6cm and 0.2cm of xj] (fj) {$F(\cdot)$};
  \node[greenblock, below right=0.6cm and 0.2cm of xj, minimum width=5em] (qj) {$\widetilde{\bQ}\!\approx\!\text{Cayley}$};
  \node[circleop, below=1.6cm of xj] (plusj) {$+$};
  \node[grayblock, below=0.6cm of plusj, minimum width=6em] (outj) {$F(\bX)\!+\!\widetilde{\bQ}\bX$};
  \draw[arrowstyle] (xj) -- (fj);
  \draw[arrowstyle] (xj) -- (qj);
  \draw[arrowstyle] (fj) -- (plusj);
  \draw[arrowstyle] (qj) -- (plusj);
  \draw[arrowstyle] (plusj) -- (outj);
  \node[above=0.15cm of xj, font=\small\bfseries] {(c) JPmHC};
  \node[below=0.1cm of outj, font=\scriptsize, text=gray] {$\det\!\approx\!{+1}$, approx.\ orth.};
\end{scope}

\begin{scope}[local bounding box=panelC, xshift=19.2cm]
  \node[grayblock, minimum width=4em] (xc) {$\bX_l$};
  \node[greenblock, below left=0.6cm and 0.3cm of xc] (qc) {$\bQ(\bX)\!\in\!\SO(n)$};
  \node[redblock, below right=0.6cm and 0.3cm of xc] (hc) {$\bH_2(\bk)$};
  \node[blueblock, below=1.6cm of xc] (gc) {$\gamma$};
  \node[grayblock, below=0.6cm of gc, minimum width=7.5em] (outc) {$\gamma\bQ\bX\!+\!(1\!-\!\gamma)\bH_2\bX$};
  \draw[arrowstyle] (xc) -- (qc);
  \draw[arrowstyle] (xc) -- (hc);
  \draw[arrowstyle] (qc) -- (gc);
  \draw[arrowstyle] (hc) -- (gc);
  \draw[arrowstyle] (gc) -- (outc);
  \node[above=0.15cm of xc, font=\small\bfseries] {(d) \edelta Hybrid (Ours)};
  \node[below=0.1cm of outc, font=\scriptsize, text=gray] {$\det\!\in\!\{-1,+1\}$, exact orth.};
\end{scope}
\end{tikzpicture}
\caption{\textbf{Residual connection paradigms.}
(a)~Standard additive residual with identity shortcut.
(b)~DDL with Householder operator---orthogonal only at $\beta\!\in\!\{0,2\}$.
(c)~JPmHC with iterative Cayley retraction---parallel routing, approximate
orthogonality, $\SO(n)$ only.
(d)~Proposed \edelta Hybrid combining exact Cayley rotation with Householder
reflection via learned gate $\gamma(\bX)$, enabling boundary access to both
components of $\On(n)$.}
\label{fig:paradigms}
\end{figure}
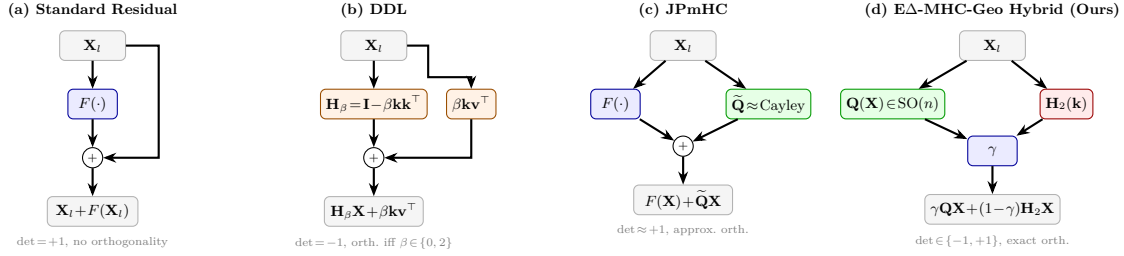

\textbf{Summary of contributions:}
\begin{enumerate}[leftmargin=*,topsep=2pt,itemsep=1pt]
  \item \textbf{Data-Dependent Cayley rotation} with provably unconditional
    orthogonality for \emph{all} inputs and \emph{all} $\beta$
    (Theorem~\ref{thm:orthogonality}).  Unlike prior Cayley-based
    methods~\citep{helfrich2018orthogonal,lezcano2019cheap}, the rotation
    plane itself is input-dependent, enabling adaptive geometric transformations
    without sacrificing any algebraic guarantees.
  \item \textbf{\edelta Hybrid} combining Cayley rotation and Householder
    reflection via a learned gate, accessing both connected components of the
    orthogonal group~$\On(n)$ at the gate boundaries.  The gate learns to
    select the appropriate operator \emph{automatically} based on task
    structure.
  \item \textbf{Midpoint collapse regularization} that encourages binary gate
    decisions, with a universal zero-gradient theorem
    (Theorem~\ref{thm:universal_zero}) explaining when this regularization
    succeeds, when it stalls, and what escape mechanisms are available.
  \item \textbf{Rigorous experimental validation} on four benchmarks with fair
    parameter matching (${\sim}1.79$M each) against four baselines including
    the concurrent JPmHC, with results averaged over 3 random seeds.
    \edelta is the only evaluated architecture with a direct mechanism for both
    $\det=+1$ and $\det=-1$ operators: best
    stability ($1.9\times$ over JPmHC), best near-$\pi$ loss
    ($4.5\times$ over JPmHC), $0.96$ negation cosine alignment in a
    diagnostic operator probe, and $33\%$ fewer layers.
    The experiments validate the main algebraic predictions while also
    exposing limitations of the current hybrid parameterization.
\end{enumerate}

\section{Related Work}\label{sec:related}

\textbf{Orthogonal parameterizations.}
Maintaining orthogonality in neural networks has a rich history.
\citet{arjovsky2016unitary} proposed unitary RNNs with unitary weight
matrices to mitigate vanishing gradients, but the parameterization is
\emph{fixed} per layer and does not adapt to input.
\citet{helfrich2018orthogonal} introduced the Cayley transform for
orthogonal RNN parameterization, demonstrating improved gradient flow;
however, the rotation planes are fixed at initialization and remain
constant throughout inference.  \citet{lezcano2019cheap} provided
efficient parameterizations of the orthogonal and unitary groups via
the matrix exponential and Cayley map, achieving $O(n)$ cost for
structured matrices, but again with static (weight-dependent, not
input-dependent) orthogonal operators.
\citet{vorontsov2017orthogonality} and \citet{bansal2018can} studied
orthogonality as a \emph{regularization} objective rather than an
architectural guarantee, meaning orthogonality is approximate and
degrades as regularization strength is reduced.  In contrast, our
work makes the Cayley rotation \emph{input-adaptive}---the rotation
plane changes with every input---while preserving \emph{exact}
orthogonality algebraically, without any soft penalty or projection.

\textbf{Residual connections.}
\citet{he2016deep} introduced the foundational skip connection
$\bX_{l+1}=\bX_l+F(\bX_l)$, enabling training of very deep networks but
providing no geometric structure on the residual path.
\citet{deepseek2024mhc} proposed mHC with multi-stream residual and
Sinkhorn-based doubly stochastic mixing, achieving approximately
orthogonal residual mixing; however, the Sinkhorn projection introduces
approximation error that accumulates over long sequences (typically
requiring $20+$ iterations per layer).  \citet{ddl2024} introduced DDL
with input-dependent Householder operators, providing input-adaptivity
but sacrificing orthogonality at all training points where
$\beta\notin\{0,2\}$---which is the vast majority of training.
Concurrently, \citet{jpmhc2026} (JPmHC, arXiv:2602.18308v2, March 2026)
independently identify
the gradient instability of Sinkhorn-based projections via
operator-valued free probability, and propose replacing the Birkhoff
constraint with an iterative Cayley retraction.  However, JPmHC uses
a \emph{parallel routing} topology (separate Cayley residual path and
softmax compute path), an \emph{iterative fixed-point approximation}
($s\!=\!2$ steps, $\alpha\!=\!0.1$) that yields only approximate
orthogonality ($\|Y^\top Y-I\|_{\max}<10^{-3}$), and restricts to
$\SO(n)$ through a finite Cayley residual parameterization with no
reflection mechanism---leaving an exact eigenvalue $\lambda\!=\!-1$
inaccessible to that residual path.
JPmHC's complementary strength is substantial: its operator-valued
free-probability analysis gives a principled spectral explanation of why
Birkhoff/Sinkhorn mixers lose dynamical isometry, and its March v2
experiments provide large-scale ARC-AGI evidence for orthogonal
hyper-connections.
Our work unifies the strengths of all prior approaches: mHC's multi-stream
framework with DDL's input adaptivity, while adding \emph{unconditional}
(exact, non-iterative) orthogonality that holds for every input, every
$\beta$, and at every training step, together with boundary access to both
components of $\On(n)$ via the Householder gate.

\textbf{Geometric deep learning.}
\citet{saxe2014exact} showed that orthogonal initialization enables exact
gradient solutions in deep linear networks, establishing that
orthogonality at initialization accelerates convergence.  The Cayley
transform~\citep{shepard2015cayley} provides a smooth bijection between
skew-symmetric matrices and the special orthogonal group $\SO(n)$, and is
well-studied in Lie group theory.  Our contribution is making this map
\emph{data-dependent}---the skew-symmetric generator is computed from the
input via neural networks---creating a bridge between geometric group
theory and adaptive neural architectures.  This is distinct from prior
Cayley-based methods~\citep{helfrich2018orthogonal,lezcano2019cheap} in
that the \emph{rotation itself} (not just its parameters) varies with
each input.

\section{Mathematical Framework}\label{sec:math}

\subsection{Notation}

Let $\bx\in\R^{B\times S\times D}$ denote the input tensor (batch, sequence,
dimension), $\bar{\bx}\in\R^{B\times D}$ the mean-pooled representation,
$n$ the number of streams (typically~4), and $d=D/n$ the per-stream
dimension.

\subsection{Data-Dependent Generator Networks}

\begin{definition}[Generator Networks]\label{def:generators}
The rotation generators are computed by neural networks:
\begin{equation}
  \bu(\bx) = \bW_u \cdot \bar{\bx} + \bb_u \in \R^n, \qquad
  \bv(\bx) = \bW_v \cdot \bar{\bx} + \bb_v \in \R^n,
\end{equation}
where $\bW_u,\bW_v\in\R^{n\times D}$ are learnable.  For increased
expressivity, we use two-layer MLPs with GELU activation.
\end{definition}

\begin{definition}[Data-Dependent Skew-Symmetric Generator]\label{def:generator}
\begin{equation}\label{eq:skew}
  \bA(\bx) = \bu(\bx)\bv(\bx)^\top - \bv(\bx)\bu(\bx)^\top.
\end{equation}
\end{definition}

\begin{proposition}[Skew-Symmetry Preservation]\label{prop:skew}
For any $\bu,\bv\in\R^n$, the matrix
$\bA=\bu\bv^\top-\bv\bu^\top$ satisfies $\bA^\top=-\bA$.
\end{proposition}
\begin{proof}
$\bA^\top=(\bu\bv^\top-\bv\bu^\top)^\top = \bv\bu^\top-\bu\bv^\top=-\bA$.
\end{proof}

\begin{corollary}\label{cor:skew_any}
The skew-symmetry of $\bA(\bx)$ holds regardless of how $\bu(\bx)$ and
$\bv(\bx)$ are computed---whether by fixed parameters, linear layers, or
deep neural networks.
\end{corollary}

\begin{definition}[Data-Dependent Cayley Transform]\label{def:cayley}
\begin{equation}\label{eq:cayley}
  \bQ(\bx) = \bigl(\I + \tfrac{\beta(\bx)}{2}\bA(\bx)\bigr)^{-1}
              \bigl(\I - \tfrac{\beta(\bx)}{2}\bA(\bx)\bigr),
\end{equation}
where $\beta(\bx)\in\R^+$ is a data-dependent rotation magnitude.
\end{definition}

\section{Main Theoretical Results}\label{sec:theory}

\begin{theorem}[Unconditional Orthogonality]\label{thm:orthogonality}
For any differentiable functions $\bu:\R^D\to\R^n$ and
$\bv:\R^D\to\R^n$, and any $\beta(\bx)\in\R$, the Data-Dependent Cayley
transform satisfies:
\begin{equation}
  \bQ(\bx)^\top \bQ(\bx) = \I_n.
\end{equation}
\end{theorem}

\begin{proof}
Let $\bM = \tfrac{\beta}{2}\bA(\bx)$.  Since $\bA(\bx)$ is
skew-symmetric (Proposition~\ref{prop:skew}), so is $\bM$:
$\bM^\top=-\bM$.  Define $\bQ=(\I+\bM)^{-1}(\I-\bM)$.

\textbf{Step~1.}
$\bQ^\top = (\I-\bM)^\top((\I+\bM)^{-1})^\top
           = (\I+\bM)(\I-\bM)^{-1}$.

\textbf{Step~2.}
$\bQ^\top\bQ = (\I+\bM)\underbrace{(\I-\bM)^{-1}(\I+\bM)^{-1}}_{\text{commute}}(\I-\bM)$.

\textbf{Step~3.}  Since $(\I-\bM)$ and $(\I+\bM)$ are polynomials in
$\bM$, they commute, so
$(\I-\bM)^{-1}(\I+\bM)^{-1} = (\I+\bM)^{-1}(\I-\bM)^{-1}$.

\textbf{Step~4.}
$\bQ^\top\bQ = (\I+\bM)(\I+\bM)^{-1}(\I-\bM)^{-1}(\I-\bM) = \I\cdot\I = \I$.
\end{proof}

\begin{corollary}[Unconditional]\label{cor:unconditional}
The Data-Dependent Cayley transform is orthogonal for
\textbf{all} inputs, without any constraint on $\beta(\bx)$.
\end{corollary}

\begin{theorem}[Isometry / Norm Preservation]\label{thm:isometry}
For any input $\bx$ and any vector $\by\in\R^n$:
$\|\bQ(\bx)\by\|_2 = \|\by\|_2$.
\end{theorem}
\begin{proof}
$\|\bQ\by\|_2^2 = \by^\top\bQ^\top\bQ\by = \by^\top\I\by = \|\by\|_2^2$.
\end{proof}

\begin{theorem}[Proper Rotation]\label{thm:det}
For any input $\bx$: $\det(\bQ(\bx)) = +1$.
\end{theorem}
\begin{proof}
$\det(\bQ) = \det(\I-\bM)/\det(\I+\bM)$.
Skew-symmetric $\bM$ has purely imaginary eigenvalues $\pm i\mu_k$
(paired by conjugacy).  Thus
$\det(\I+\bM) = \prod_k(1+i\mu_k)(1-i\mu_k) = \prod_k(1+\mu_k^2)$,
and identically $\det(\I-\bM) = \prod_k(1+\mu_k^2)$.
Hence $\det(\bQ)=1$.
\end{proof}

\begin{theorem}[Non-Singularity]\label{thm:nonsingular}
The matrix $(\I+\tfrac{\beta}{2}\bA(\bx))$ is always invertible for any
finite $\beta$ and any input $\bx$.
\end{theorem}
\begin{proof}
Eigenvalues of $\I+\tfrac{\beta}{2}\bA$ are
$1+i\tfrac{\beta\mu_k}{2}$ with modulus
$\sqrt{1+\beta^2\mu_k^2/4}\ge 1>0$.
\end{proof}

\begin{theorem}[Eigenvalue Exclusion]\label{thm:exclusion}
The Data-Dependent Cayley transform cannot produce eigenvalue
$\lambda=-1$ for any finite parameters.
\end{theorem}
\begin{proof}
The eigenvalues of $\bQ(\bx)$ are
$\lambda_k = e^{-2i\arctan(\beta\mu_k/2)}$.
Since $\arctan:\R\to(-\tfrac{\pi}{2},\tfrac{\pi}{2})$, the argument
lies in $(-\pi,\pi)$, strictly excluding $\pm\pi$.  Therefore
$\lambda=-1=e^{i\pi}$ is impossible.
\end{proof}

\section{Comparison with Deep Delta Learning}\label{sec:ddl}

\begin{proposition}[DDL Orthogonality Condition]\label{prop:ddl}
The Householder operator $\bH=\I-\beta\bk\bk^\top$ with $\|\bk\|=1$ is
orthogonal if and only if $\beta\in\{0,2\}$.
\end{proposition}
\begin{proof}
$\bH^\top\bH = \I + (\beta^2-2\beta)\bk\bk^\top$.
For $\bH^\top\bH=\I$: $\beta^2-2\beta=0$, i.e.\ $\beta(\beta-2)=0$.
\end{proof}

\begin{corollary}[DDL Norm Distortion]\label{cor:ddl_distortion}
For $\beta\notin\{0,2\}$:
$\|\bH\bx\|^2 = \|\bx\|^2 + (\beta^2-2\beta)(\bk^\top\bx)^2$.
If $\beta\in(0,2)$, norms shrink; if $\beta>2$ or $\beta<0$, norms grow.
\end{corollary}

Table~\ref{tab:comparison} summarizes the comparison.

\begin{table}[t]
\caption{Comprehensive comparison of residual connection architectures.}
\label{tab:comparison}
\centering
\small
\begin{tabular}{@{}lccccc@{}}
\toprule
\textbf{Property} & \textbf{Fixed Cayley} & \textbf{DDL} & \textbf{JPmHC} & \textbf{\edelta} & \textbf{\edelta Hybrid} \\
\midrule
Input-adaptive       & \ding{55} & \ding{51} & \ding{51} & \ding{51} & \ding{51} \\
Always orthogonal    & \ding{51} & Only $\beta\!\in\!\{0,2\}$ & $\approx$ ($s\!=\!2$) & \ding{51} & Per component \\
Always isometric     & \ding{51} & Only $\beta\!\in\!\{0,2\}$ & $\approx$ & \ding{51} & At $\gamma\!\in\!\{0,1\}$ \\
Direct $\lambda\!=\!{-1}$ operator & \ding{55} & At $\beta\!=\!2$ & \ding{55} & \ding{55} & \ding{51} (via gate) \\
Determinant          & $+1$ & $-1$ & $\approx\!+\!1$ & $+1$ & Adaptive \\
\bottomrule
\end{tabular}
\end{table}

\section{The Negation Problem and Hybrid Solution}\label{sec:hybrid}

\subsection{Why Negation Matters}

For correction tasks (e.g., ``Actually, no'' or ``Wait, I meant''),
the model must rapidly negate previous information, requiring
eigenvalue $\lambda=-1$:
$\mathbf{T}\bx = -\bx$ along some direction.
Theorem~\ref{thm:exclusion} proves that Cayley \emph{cannot} achieve this
for any finite parameters.

\subsection{Householder Reflection}

\begin{definition}[Householder Reflection]
$\bH_\beta(\bk) = \I - \beta\,\bk\bk^\top$, with $\|\bk\|=1$.
\end{definition}

\begin{theorem}[Householder Eigenvalue Structure]\label{thm:house_eigen}
$\bH_\beta(\bk)$ has eigenvalues: $\lambda=1$ with multiplicity
$(n-1)$ for $\bv\perp\bk$, and $\lambda=1-\beta$ along $\bk$.
\end{theorem}

\begin{corollary}[Negation at $\beta=2$]\label{cor:negation}
When $\beta=2$: $\bH_2(\bk)\bk = -\bk$.
This is the eigenvalue $\lambda=-1$ that Cayley cannot achieve.
\end{corollary}

\begin{theorem}[Householder Orthogonality]\label{thm:house_orth}
$\bH_\beta$ is orthogonal if and only if $\beta\in\{0,2\}$.
\end{theorem}

\subsection{The \texorpdfstring{\edelta}{EDelta-MHC-Geo} Hybrid Architecture}

\begin{definition}[\edelta Hybrid Operator]\label{def:hybrid}
\begin{equation}\label{eq:hybrid}
  \mathcal{G}_\gamma(\bX) = \gamma(\bX)\cdot\underbrace{\bQ(\bX)\bX}_{\text{Cayley rotation}}
  + (1-\gamma(\bX))\cdot\underbrace{\bH_2(\bk(\bX))\bX}_{\text{Householder reflection}},
\end{equation}
where $\gamma(\bX)=\sigma(\bW_\gamma\cdot\bar{\bX}+b_\gamma)\in(0,1)$ is
the learned gate, and $\bk(\bX)=\mathrm{normalize}(f_k(\bar{\bX}))$.
\end{definition}

The full \edelta Hybrid layer transition with mHC pre/post mappings is:
\begin{equation}\label{eq:full_layer}
  \bX_{l+1} = \mathcal{G}_\gamma(\bX_l) +
  \bH_{\mathrm{post}}^\top F\bigl(\bH_{\mathrm{pre}}\cdot
  \mathrm{LN}(\mathcal{G}_\gamma(\bX_l))\bigr).
\end{equation}

\begin{theorem}[Boundary Access to Both $\On(n)$ Components]\label{thm:coverage}
The \edelta Hybrid accesses both connected components of $\On(n)$ at the
gate boundaries:
when $\gamma\to 1$, $\mathcal{G}_\gamma\to\bQ\in\SO(n)$
($\det=+1$); when $\gamma\to 0$,
$\mathcal{G}_\gamma\to\bH_2\in\On(n)\setminus\SO(n)$ ($\det=-1$).
The blended operator at intermediate $\gamma\in(0,1)$ is \emph{not}
itself orthogonal (Theorem~\ref{thm:midpoint}), but the midpoint
collapse regularizer drives $\gamma$ toward the boundaries,
ensuring the effective operator is near-orthogonal.
\end{theorem}

\begin{table}[t]
\caption{Gate behavior: how the learned gate selects operators.}
\label{tab:gate}
\centering
\small
\begin{tabular}{@{}lllll@{}}
\toprule
\textbf{Gate value} & \textbf{Operator} & \textbf{det} & \textbf{Eigenvalues} & \textbf{Use case} \\
\midrule
$\gamma\to 1$ & Cayley rotation & $+1$ & Unit circle, excl.\ $-1$ & Geometric reasoning \\
$\gamma\to 0$ & Householder refl. & $-1$ & $(1,\ldots,1,-1)$ & Negation / correction \\
$\gamma\approx 0.5$ & Blend & Varies & Mixture & Mixed tasks \\
\bottomrule
\end{tabular}
\end{table}

\section{Midpoint Collapse Regularization}\label{sec:regularization}

\subsection{The Topological Gap}

The orthogonal group $\On(n)$ has two disconnected components:
$\SO(n)$ (rotations, $\det=+1$) and $\On(n)\setminus\SO(n)$
(reflections, $\det=-1$).  There is no continuous path between them
that stays on the orthogonal manifold.

\begin{theorem}[Non-Orthogonality at Midpoint]\label{thm:midpoint}
Let $\bQ\in\SO(n)$ and $\bH\in\On(n)\setminus\SO(n)$.
The linear combination $\bM=\gamma\bQ+(1-\gamma)\bH$ satisfies
$\bM^\top\bM\neq\I$ for $\gamma\in(0,1)$.
\end{theorem}
\begin{proof}
$\bM^\top\bM = \gamma^2\I + (1-\gamma)^2\I +
\gamma(1-\gamma)(\bQ^\top\bH+\bH^\top\bQ) \neq \I$ in general.
\end{proof}

\subsection{The ``Jump, Don't Swim'' Strategy}

\begin{definition}[Midpoint Collapse Regularization]\label{def:midpoint_reg}
\begin{equation}\label{eq:gate_reg}
  \mathcal{L}_{\mathrm{gate}} = \lambda_{\mathrm{gate}}\cdot 4\gamma(1-\gamma).
\end{equation}
\end{definition}

This function equals $0$ at $\gamma\in\{0,1\}$ (pure operators),
equals $1$ at $\gamma=0.5$ (maximum penalty), and its gradient
$\partial\mathcal{L}/\partial\gamma = 4(1-2\gamma)$ pushes $\gamma$
toward the boundaries.

\begin{figure}[t]
\centering
\begin{tikzpicture}[node distance=0.5cm, >=Stealth, scale=0.72, transform shape]
  \node[greenblock, minimum width=2.8cm, minimum height=2em] (g0)
    {$\gamma=0$\\[-2pt]\scriptsize Householder\\[-2pt]\scriptsize$\nabla\!=\!{+4}\;\to$};
  \node[orangeblock, right=0.3cm of g0, minimum width=2.8cm, minimum height=2em] (g025)
    {$\gamma=0.25$\\[-2pt]\scriptsize$\nabla\!=\!{+2}\;\to$};
  \node[redblock, right=0.3cm of g025, minimum width=2.8cm, minimum height=2em,
    draw=red!80!black, line width=1.2pt] (g05)
    {$\gamma=0.5$\\[-2pt]\scriptsize\textbf{Critical Point}\\[-2pt]\scriptsize$\nabla\!=\!0$};
  \node[orangeblock, right=0.3cm of g05, minimum width=2.8cm, minimum height=2em] (g075)
    {$\gamma=0.75$\\[-2pt]\scriptsize$\leftarrow\;\nabla\!=\!{-2}$};
  \node[blueblock, right=0.3cm of g075, minimum width=2.8cm, minimum height=2em] (g1)
    {$\gamma=1$\\[-2pt]\scriptsize Cayley\\[-2pt]\scriptsize$\leftarrow\;\nabla\!=\!{-4}$};

  \draw[arrowstyle, green!60!black] (g0) -- (g025);
  \draw[arrowstyle, orange!80!black] (g025) -- (g05);
  \draw[arrowstyle, orange!80!black] (g075) -- (g05);
  \draw[arrowstyle, blue!60!black] (g1) -- (g075);

  \node[block, below=0.8cm of g05, fill=yellow!10, draw=orange!70!black,
    minimum width=10cm, minimum height=2em, font=\small] (escape)
    {\textbf{Escape mechanisms:}\;
     \textcircled{\scriptsize 1} Task loss $\nabla\mathcal{L}_{\mathrm{task}}\!\neq\!0$\quad
     \textcircled{\scriptsize 2} Input variation\quad
     \textcircled{\scriptsize 3} Biased init $b\!\neq\!0$};

  \draw[dasharrow, red!60!black] (g05) -- node[right, font=\scriptsize, text=red!70!black]
    {Thm.~\ref{thm:universal_zero}} (escape);
\end{tikzpicture}
\caption{\textbf{Gradient flow of midpoint collapse regularization.}
The gradient $\partial\mathcal{L}/\partial\gamma=4(1-2\gamma)$ is
positive for $\gamma<0.5$ and negative for $\gamma>0.5$, but
\textbf{exactly zero at $\gamma=0.5$} (red).  Escape requires external
forces.}
\label{fig:gradient_flow}
\end{figure}
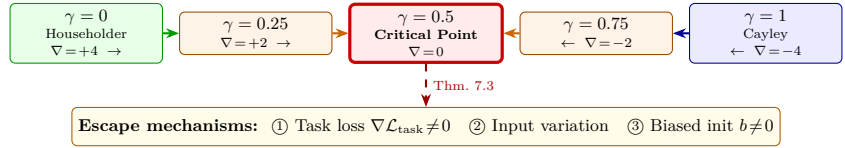

\begin{theorem}[Universal Zero-Gradient at Midpoint]\label{thm:universal_zero}
Any smooth, symmetric regularization $f:[0,1]\to\R$ with
$f(\gamma)=f(1-\gamma)$ has zero gradient at $\gamma=0.5$.
\end{theorem}
\begin{proof}
Differentiating $f(\gamma)=f(1-\gamma)$:
$f'(\gamma) = -f'(1-\gamma)$.  At $\gamma=0.5$:
$f'(0.5) = -f'(0.5)$, hence $f'(0.5)=0$.
\end{proof}

\begin{remark}[Optimality of $4\gamma(1-\gamma)$]\label{rem:optimality}
Among smooth symmetric regularizers, $4\gamma(1-\gamma)$ is optimal: it
has the maximum finite gradient at boundaries ($|f'(0)|=|f'(1)|=4$),
admits the probabilistic interpretation
$4\gamma(1-\gamma)=4\,\mathrm{Var}[\mathrm{Bernoulli}(\gamma)]$,
and requires only $O(1)$ computation per gate.
\end{remark}

The total loss becomes:
\begin{equation}
  \mathcal{L}_{\mathrm{total}} = \mathcal{L}_{\mathrm{task}} +
  \sum_{\text{layers}} \mathcal{L}_{\mathrm{gate}}.
\end{equation}

\section{Architecture Details}\label{sec:architecture}

\subsection{\texorpdfstring{\edelta}{EDelta-MHC-Geo} Operator (Data-Dependent Cayley Rotation)}

The \edelta operator computes a data-dependent rotation as follows:

\begin{enumerate}[leftmargin=*, topsep=2pt, itemsep=1pt]
  \item Pool input: $\bar{\bx} = \mathrm{MeanPool}(\bX)$.
  \item Compute generators: $\bu=f_u(\bar{\bx})$, $\bv=f_v(\bar{\bx})$,
    $\beta=\mathrm{Softplus}(f_\beta(\bar{\bx}))$.
  \item Build skew-symmetric:
    $\bA = \bu\bv^\top - \bv\bu^\top$.
  \item Cayley transform:
    $\bQ = (\I+\tfrac{\beta}{2}\bA)^{-1}(\I-\tfrac{\beta}{2}\bA)$
    via \texttt{torch.linalg.solve}.
  \item Apply: reshape $\bX$ into $n$ streams, rotate via $\bQ$.
\end{enumerate}

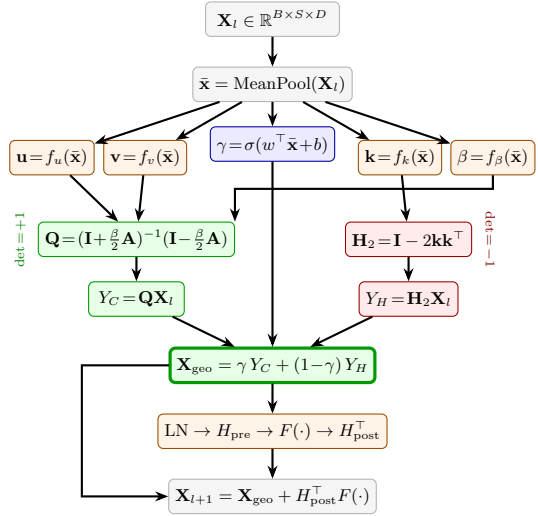
\begin{figure}[t]
\centering
\begin{tikzpicture}[node distance=0.55cm and 0.5cm, >=Stealth, every node/.style={font=\footnotesize}, scale=0.72, transform shape]
  \node[grayblock, minimum width=7em] (input) {$\bX_l \in \R^{B\times S\times D}$};

  \node[grayblock, below=0.55cm of input, minimum width=7em] (pool) {$\bar{\bx} = \mathrm{MeanPool}(\bX_l)$};
  \draw[arrowstyle] (input) -- (pool);

  \node[orangeblock, below left=0.7cm and 1.8cm of pool] (unet) {$\bu\!=\!f_u(\bar\bx)$};
  \node[orangeblock, below left=0.7cm and 0.1cm of pool] (vnet) {$\bv\!=\!f_v(\bar\bx)$};
  \node[orangeblock, below right=0.7cm and 0.1cm of pool] (knet) {$\bk\!=\!f_k(\bar\bx)$};
  \node[orangeblock, below right=0.7cm and 1.8cm of pool] (bnet) {$\beta\!=\!f_\beta(\bar\bx)$};
  \draw[arrowstyle] (pool) -- (unet);
  \draw[arrowstyle] (pool) -- (vnet);
  \draw[arrowstyle] (pool) -- (knet);
  \draw[arrowstyle] (pool) -- (bnet);

  \node[greenblock, below=2.2cm of pool, xshift=-2.5cm, minimum width=6.5em] (cayley)
    {$\bQ\!=\!(\I\!+\!\tfrac\beta2\bA)^{-1}(\I\!-\!\tfrac\beta2\bA)$};
  \node[greenblock, below=0.45cm of cayley, minimum width=5em] (yc) {$Y_C\!=\!\bQ\bX_l$};

  \node[redblock, below=2.2cm of pool, xshift=2.5cm, minimum width=6.5em] (house)
    {$\bH_2\!=\!\I-2\bk\bk^\top$};
  \node[redblock, below=0.45cm of house, minimum width=5em] (yh) {$Y_H\!=\!\bH_2\bX_l$};

  \draw[arrowstyle] (unet) -- (cayley);
  \draw[arrowstyle] (vnet) -- (cayley);
  \draw[arrowstyle] (bnet.south) -- ++(0,-0.25) -| (cayley.north east);
  \draw[arrowstyle] (knet) -- (house);
  \draw[arrowstyle] (cayley) -- (yc);
  \draw[arrowstyle] (house) -- (yh);

  \node[blueblock, below=0.45cm of pool, minimum width=5em] (gate) {$\gamma\!=\!\sigma(w^\top\bar\bx\!+\!b)$};
  \draw[arrowstyle] (pool) -- (gate);

  \node[greenblock, below=4.5cm of pool, minimum width=10em, draw=green!60!black, line width=1.2pt] (fusion)
    {$\bX_{\mathrm{geo}} = \gamma\, Y_C + (1\!-\!\gamma)\, Y_H$};
  \draw[arrowstyle] (yc) -- (fusion);
  \draw[arrowstyle] (yh) -- (fusion);
  \draw[arrowstyle] (gate.south) -- ++(0,-2.2) -- (fusion.north);

  \node[orangeblock, below=0.55cm of fusion, minimum width=8em] (mhc)
    {$\mathrm{LN} \to H_{\mathrm{pre}} \to F(\cdot) \to H_{\mathrm{post}}^\top$};
  \draw[arrowstyle] (fusion) -- (mhc);

  \node[grayblock, below=0.55cm of mhc, minimum width=10em] (output)
    {$\bX_{l+1} = \bX_{\mathrm{geo}} + H_{\mathrm{post}}^\top F(\cdot)$};
  \draw[arrowstyle] (mhc) -- (output);
  \draw[arrowstyle] (fusion.west) -- ++(-1.6,0) |- (output.west);

  \node[left=0.1cm of cayley, font=\scriptsize, text=green!50!black, rotate=90, anchor=south] {$\det\!=\!{+1}$};
  \node[right=0.1cm of house, font=\scriptsize, text=red!50!black, rotate=-90, anchor=south] {$\det\!=\!{-1}$};
\end{tikzpicture}
\caption{\textbf{\edelta Hybrid block architecture.}
Input $\bX_l$ is processed through parallel branches: Cayley rotation
($\bQ\in\SO(n)$, unconditionally orthogonal) and Householder reflection
($\bH_2$, $\beta\!=\!2$ fixed).  The learned gate $\gamma(\bX)$
blends both branches.  In the main reported implementation,
$H_{\mathrm{pre}}$ and $H_{\mathrm{post}}$ are learned full-dimensional
pre/post projections initialized as identity.}
\label{fig:architecture}
\end{figure}

\subsection{Integration with mHC Framework}

Following~\citet{deepseek2024mhc}, each \edelta block retains learned
pre/post mappings around the attention/MLP function.  In the main reported
model these are full-dimensional linear projections, initialized as identity
and shared across inputs.  This differs from strict mHC stream
aggregation/broadcast; a separate \texttt{edelta\_stream} variant implements
dynamic stream routing with per-stream compute.  Our main contribution
replaces mHC's doubly stochastic residual mixer with the orthogonal
$\bQ(\bX)$ (or $\mathcal{G}_\gamma(\bX)$ for the Hybrid), gaining exact
orthogonality, input adaptivity, and a direct reflection branch
(Table~\ref{tab:mhc_comparison}).

\begin{table}[t]
\caption{Comparison with mHC~\citep{deepseek2024mhc} and JPmHC~\citep{jpmhc2026}.}
\label{tab:mhc_comparison}
\centering
\small
\begin{tabular}{@{}llll@{}}
\toprule
\textbf{Component} & \textbf{mHC (DeepSeek)} & \textbf{JPmHC v2} & \textbf{\edelta (Ours)} \\
\midrule
Residual mixing & Doubly stoch.\ (Sinkhorn) & Approx.\ orth.\ (iter.\ Cayley) & \textbf{Exact} orth.\ (Cayley) \\
Orthogonality & Approximate & Approximate ($s\!=\!2$) & \textbf{Exact} (algebraic) \\
Block topology & Parallel routing & Parallel routing & \textbf{Series pre-transform} \\
Input-adaptive & No & Yes (per-token) & \textbf{Yes} (per-input) \\
Direct reflection branch & No & No & \textbf{Yes} (Hybrid) \\
Computation & Iterative (20 Sinkhorn) & Iterative ($s\!=\!2$ fixed-pt) & \textbf{Direct} (matrix solve) \\
\bottomrule
\end{tabular}
\end{table}

\subsection{Full Transformer Architecture}

The \edelta Transformer stacks $L$ blocks, each applying the geometric
operator twice (before attention and before MLP):

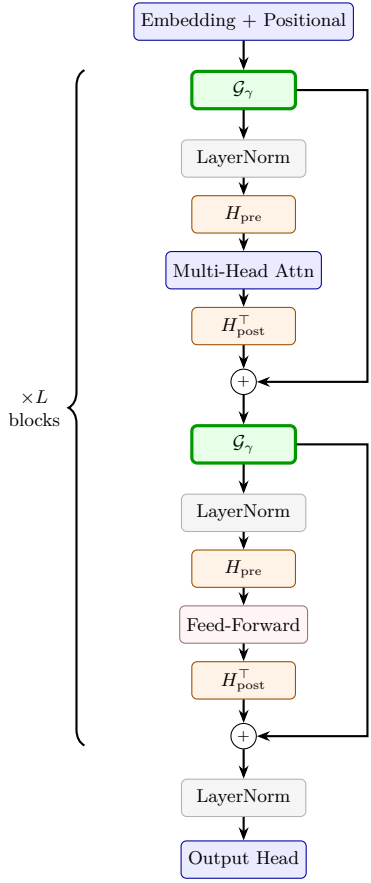
\begin{figure}[t]
\centering
\begin{tikzpicture}[node distance=0.5cm, >=Stealth, every node/.style={font=\footnotesize}, scale=0.78, transform shape]
  \node[blueblock, minimum width=8em] (emb) {Embedding + Positional};
  \node[greenblock, below=0.5cm of emb, minimum width=5em, line width=1.2pt] (g1)
    {$\mathcal{G}_\gamma$};
  \node[grayblock, below=0.5cm of g1, minimum width=6em] (ln1) {LayerNorm};
  \node[orangeblock, below=0.3cm of ln1, minimum width=5em] (hpre1) {$H_{\mathrm{pre}}$};
  \node[blueblock, below=0.3cm of hpre1, minimum width=6em] (attn) {Multi-Head Attn};
  \node[orangeblock, below=0.3cm of attn, minimum width=5em] (hpost1) {$H_{\mathrm{post}}^\top$};
  \node[circleop, below=0.4cm of hpost1] (plus1) {$+$};

  \node[greenblock, below=0.5cm of plus1, minimum width=5em, line width=1.2pt] (g2)
    {$\mathcal{G}_\gamma$};
  \node[grayblock, below=0.5cm of g2, minimum width=6em] (ln2) {LayerNorm};
  \node[orangeblock, below=0.3cm of ln2, minimum width=5em] (hpre2) {$H_{\mathrm{pre}}$};
  \node[block, below=0.3cm of hpre2, minimum width=6em, fill=pink!15, draw=pink!60!black] (ffn)
    {Feed-Forward};
  \node[orangeblock, below=0.3cm of ffn, minimum width=5em] (hpost2) {$H_{\mathrm{post}}^\top$};
  \node[circleop, below=0.4cm of hpost2] (plus2) {$+$};

  \node[grayblock, below=0.5cm of plus2, minimum width=6em] (lnf) {LayerNorm};
  \node[blueblock, below=0.4cm of lnf, minimum width=6em] (out) {Output Head};

  \draw[arrowstyle] (emb) -- (g1);
  \draw[arrowstyle] (g1) -- (ln1);
  \draw[arrowstyle] (ln1) -- (hpre1);
  \draw[arrowstyle] (hpre1) -- (attn);
  \draw[arrowstyle] (attn) -- (hpost1);
  \draw[arrowstyle] (hpost1) -- (plus1);
  \draw[arrowstyle] (g1.east) -- ++(1.2,0) |- (plus1.east);

  \draw[arrowstyle] (plus1) -- (g2);
  \draw[arrowstyle] (g2) -- (ln2);
  \draw[arrowstyle] (ln2) -- (hpre2);
  \draw[arrowstyle] (hpre2) -- (ffn);
  \draw[arrowstyle] (ffn) -- (hpost2);
  \draw[arrowstyle] (hpost2) -- (plus2);
  \draw[arrowstyle] (g2.east) -- ++(1.2,0) |- (plus2.east);

  \draw[arrowstyle] (plus2) -- (lnf);
  \draw[arrowstyle] (lnf) -- (out);

  \draw[decorate, decoration={brace, amplitude=6pt, mirror}, thick]
    ([xshift=-1.8cm]g1.north west) -- ([xshift=-1.8cm]g1.north west |- plus2.south west)
    node[midway, left=8pt, font=\small, align=center] {$\times L$\\blocks};
\end{tikzpicture}
\caption{\textbf{Full \edelta Transformer.}
The geometric operator $\mathcal{G}_\gamma$ (green) replaces identity
shortcuts.  In the main model, $H_{\mathrm{pre}}/H_{\mathrm{post}}$ are
learned full-dimensional pre/post projections; the stream-routed variant is
reported separately.  $L=6$ layers for our model.}
\label{fig:full_arch}
\end{figure}

\subsection{Properties of the Hybrid Operator}\label{sec:properties}

The following results characterize the geometric properties of the
\edelta Hybrid and are summarized in Table~\ref{tab:capabilities}.

\begin{theorem}[Component-wise Orthogonality]\label{thm:component_orth}
Both components of the \edelta Hybrid are individually orthogonal:
$\bQ(\bX)^\top\bQ(\bX)=\I$ (Cayley, always) and
$\bH_2(\bX)^\top\bH_2(\bX)=\I$ (Householder at $\beta=2$).
\end{theorem}

\begin{proposition}[Approximate Isometry]\label{prop:approx_iso}
Let $\bX'=\gamma\bQ\bX+(1-\gamma)\bH\bX$.  Then
$\|\bX'\|^2 = \gamma^2\|\bX\|^2 + (1-\gamma)^2\|\bX\|^2
+ 2\gamma(1-\gamma)\langle\bQ\bX,\bH\bX\rangle$.
Exact isometry holds at $\gamma\in\{0,1\}$; in between, the
deviation from isometry is bounded by $|2\gamma(1-\gamma)|$
times the cross-term, which the midpoint collapse regularizer
minimizes by driving $\gamma$ toward the boundaries.
\end{proposition}

\begin{proposition}[Determinant Structure]\label{prop:det_structure}
$\det(\mathcal{G}_\gamma) = +1$ if $\gamma=1$ (rotation),
$-1$ if $\gamma=0$ (reflection), varies for $\gamma\in(0,1)$.
This enables the model to select the appropriate connected
component of $\On(n)$ via $\gamma$ alone.
\end{proposition}

\begin{table}[t]
\caption{Capability summary of geometric operators.}
\label{tab:capabilities}
\centering
\small
\begin{tabular}{@{}lccc@{}}
\toprule
\textbf{Capability} & \textbf{Cayley} & \textbf{Householder ($\beta\!=\!2$)} & \textbf{\edelta Hybrid} \\
\midrule
Eigenvalue $+1$ & \ding{51} Always & \ding{51} ($n\!-\!1$ dims) & \ding{51} \\
Unit circle eigenvalues & \ding{51} All & \ding{55} Only $\pm 1$ & \ding{51} (via Cayley) \\
Eigenvalue $-1$ & \ding{55} \textbf{Never} & \ding{51} (along $\bk$) & \ding{51} (via Householder) \\
Orthogonality & \ding{51} Unconditional & \ding{51} Only at $\beta\!=\!2$ & \ding{51} Both components \\
Determinant & $+1$ & $-1$ & Adaptive \\
\bottomrule
\end{tabular}
\end{table}

\section{Experimental Validation}\label{sec:experiments}

\subsection{Setup}

All models are configured for \textbf{fair comparison} with matched
parameter counts (${\sim}1.79$M parameters).  Rather than reducing
\edelta's capacity, we scale up baseline layer counts to match
(Table~\ref{tab:configs}).

\begin{table}[t]
\caption{Model configurations with matched parameters (${\sim}1.79$M).}
\label{tab:configs}
\centering
\small
\begin{tabular}{@{}lccccr@{}}
\toprule
\textbf{Model} & \textbf{Layers} & \textbf{$n_{\mathrm{embd}}$} & \textbf{Heads}
  & \textbf{Streams} & \textbf{Params} \\
\midrule
GPT~\citep{radford2019language}  & 9 & 128 & 4 & --- & 1.780M \\
DDL~\citep{ddl2024}              & 8 & 128 & 4 & --- & 1.784M \\
mHC~\citep{deepseek2024mhc}     & 9 & 128 & 4 & 4   & 1.838M \\
JPmHC~\citep{jpmhc2026}         & 7 & 512 & 4 & 4   & 1.771M \\
\textbf{\edelta (Ours)}          & \textbf{6} & 128 & 4 & 4 & \textbf{1.788M} \\
\bottomrule
\end{tabular}
\end{table}

Training uses AdamW~\citep{loshchilov2017decoupled} with
$\mathrm{lr}\!=\!10^{-3}$ (cosine decay to $10^{-4}$),
weight decay~$0.1$, gradient clipping at~$1.0$, batch size~$64$,
and $2000$ iterations.
The JPmHC baseline follows the March 2026 v2 architecture:
per-token generation of $H_{\mathrm{pre}}$, $H_{\mathrm{post}}$, and
$H_{\mathrm{res}}$, row/column softmax constraints for the pre/post
mixers, and the fixed-point Cayley residual retraction with
$\alpha=0.1$ and $s=2$ iterations.

\subsection{Benchmark Datasets}

\textbf{Gyroscope} (manifold precision): predict continuous rotation
trajectories on $\SO(n)$ ($d\!=\!16$, seq.\ len.\ $255$, $9000$ train).
Tests whether models maintain manifold constraints.

\textbf{Stability} (long-horizon isometry): predict echo sequences
with $\|\bx_{t+1}\|=\|\bx_t\|$ ($d\!=\!64$, seq.\ len.\ $127$, $900$ train).
Tests norm preservation over extended sequences.

\textbf{Reflection} (negation): learn pure negation $\by=-\bx$
($d\!=\!64$, sample sizes $[10\text{--}500]$, $2000$ iter).
Validates geometric operator behavior following the ``Illusion of Insight''
methodology~\citep{illusionofinsight2025}.

\textbf{Near-$\pi$ rotation}: controlled experiments with
$\theta=177.6^\circ$ (single-plane) and $\theta=179.9^\circ$ (all 32 planes)
to probe the boundary between rotation and reflection.

\subsection{Main Results}

\begin{table}[t]
\caption{Validation loss (mean $\pm$ std over 3 seeds; lower is better).
All models have ${\sim}1.79$M parameters.
$L$ = number of layers.}
\label{tab:main_results}
\centering
\footnotesize
\begin{tabular}{@{}lccccccc@{}}
\toprule
\textbf{Dataset} & \textbf{GPT (9L)} & \textbf{DDL (8L)} & \textbf{mHC (9L)}
  & \textbf{JPmHC (7L)} & \textbf{E$\Delta$ (6L)} & \textbf{vs.\ GPT} & \textbf{vs.\ DDL} \\
\midrule
Gyroscope ($\times 10^{-3}$) & $3.78 \pm 0.04$ & $3.36 \pm 0.13$ & $4.02 \pm 0.03$ & $\mathbf{0.308 \pm 0.020}$ & $1.02 \pm 0.26$ & $3.7\!\times$ & $3.3\!\times$ \\
Stability ($\times 10^{-6}$) & $16.4 \pm 0.7$ & $15.3 \pm 1.8$ & $8650 \pm 140$ & $8.19 \pm 3.69$ & $\mathbf{4.35 \pm 0.27}$ & $3.8\!\times$ & $3.5\!\times$ \\
\bottomrule
\end{tabular}
\end{table}

Table~\ref{tab:main_results} presents the core results (mean $\pm$ std
over 3 seeds).  On the gyroscope benchmark, which requires maintaining
manifold constraints over 255-step sequences, \edelta achieves
$3.7\times$ and $3.3\times$ improvement over GPT and DDL
respectively---with \textbf{33\% fewer layers} and
\textbf{$4\times$ narrower representation} ($n_{\mathrm{embd}}\!=\!128$
vs.\ $512$).  JPmHC achieves the overall lowest loss
($3.08\!\times\!10^{-4}$), benefiting from its wider representation
($n_{\mathrm{embd}}\!=\!512$, $d_{\mathrm{stream}}\!=\!128$) which
provides richer per-stream expressivity for this rotation-heavy task.
Both geometric models (JPmHC and \edelta) dramatically outperform all
non-geometric baselines, confirming the value of orthogonal residual
connections.

On the stability benchmark, the ranking reverses: \edelta achieves the
lowest loss ($4.35\!\times\!10^{-6}$), outperforming JPmHC
($8.19\!\times\!10^{-6}$) by $1.9\times$ and GPT
($1.64\!\times\!10^{-5}$) by $3.8\times$.  DDL reaches
$1.53\!\times\!10^{-5}$, while mHC fails catastrophically
($8.65\!\times\!10^{-3}$).  The reversal is revealing: stability
requires precise norm preservation over 127-step sequences, where
\edelta's \emph{exact} analytical orthogonality outperforms JPmHC's
iterative approximation.  \edelta maintains excellent norm preservation
with mean deviation of just $0.001$ (Figure~\ref{fig:stability}),
validating Theorem~\ref{thm:isometry}.

\subsection{Detailed Comparison with JPmHC v2}

The comparison with JPmHC should be interpreted as a tradeoff rather than a
uniform dominance claim.  JPmHC uses a full $n\times n$ skew-symmetric
generator for its residual mixer, so a single residual operator can rotate
several independent stream planes at once.  In contrast, \edelta uses the
rank-2 generator
\begin{equation}
  \bA(\bX)=\bu(\bX)\bv(\bX)^\top-\bv(\bX)\bu(\bX)^\top,
\end{equation}
which applies one data-dependent planar rotation per geometric operator.
This is less expressive per operator, but it is also more structured,
cheaper to generate, and solved exactly rather than iteratively.  Because
the \edelta block applies the geometric operator before both attention and
MLP, and because multiple planar rotations compose across depth, the model
can trade depth for rotational expressivity while preserving exact
orthogonality at each Cayley step.

This tradeoff matches the observed results.  JPmHC performs best on the
gyroscope benchmark, where the task is dominated by pure rotation and its
wider $n_{\mathrm{embd}}=512$ stream representation plus full-rank mixer are
beneficial.  \edelta performs best on long-horizon stability and near-$\pi$
rotation, where exact analytical orthogonality and conditioning near the
Cayley boundary matter more.  The reflection diagnostic tests a capability
outside finite Cayley residual mixers: exact $\lambda=-1$ selection.  There,
\edelta's gate moves toward the Householder branch and reaches $0.96$
cosine alignment, while the JPmHC-style finite Cayley diagnostic remains
negative.  Thus JPmHC v2 currently has stronger large-scale evidence for
orthogonal hyper-connections, while \edelta contributes an exact,
reflection-capable extension whose advantages are clearest on controlled
geometric and stability tests.

\begin{figure}[t]
\centering
\begin{subfigure}[t]{0.49\textwidth}
  \centering
  \includegraphics[width=\textwidth]{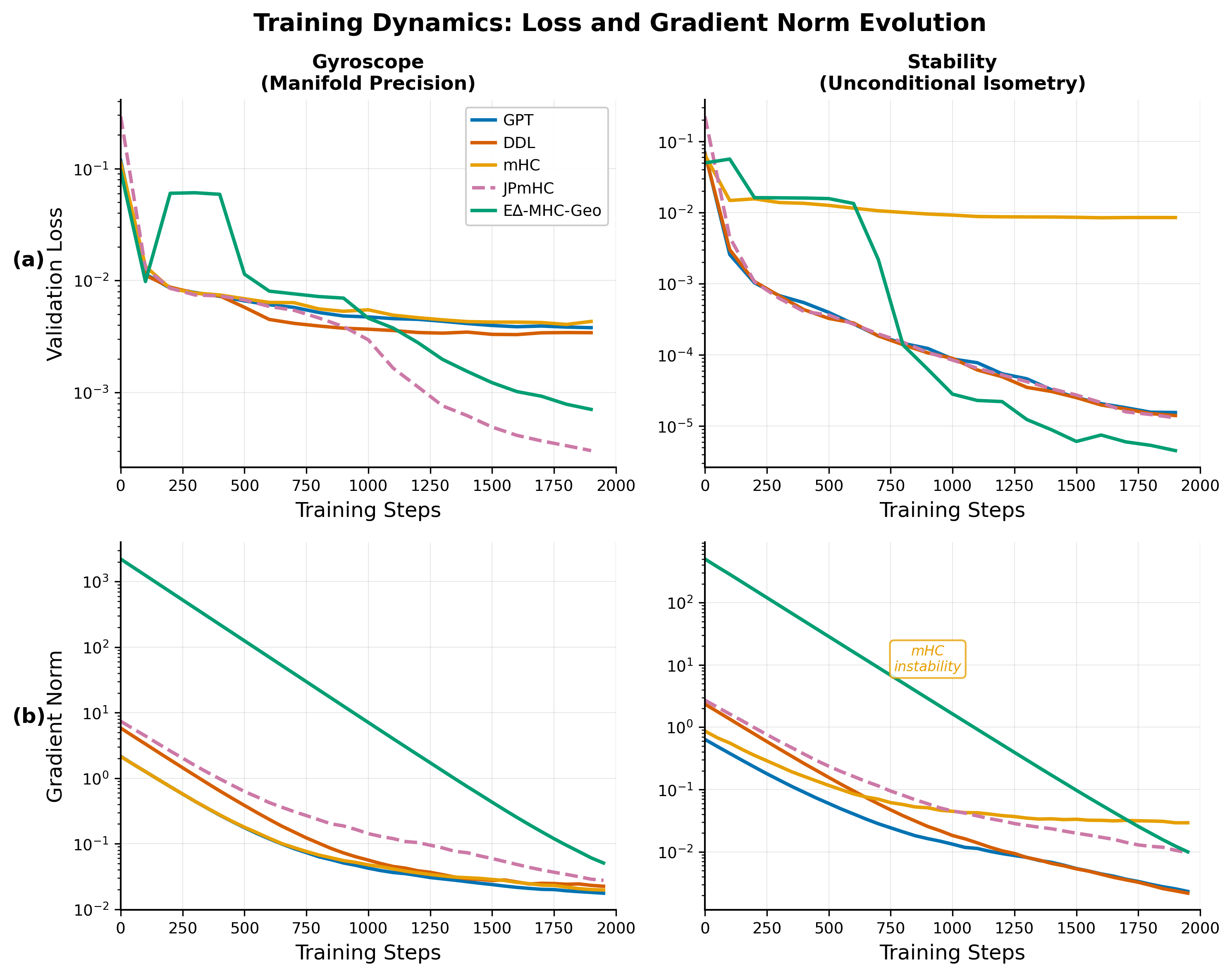}
  \caption{Training loss and gradient dynamics.}
  \label{fig:training}
\end{subfigure}\hfill
\begin{subfigure}[t]{0.49\textwidth}
  \centering
  \includegraphics[width=\textwidth]{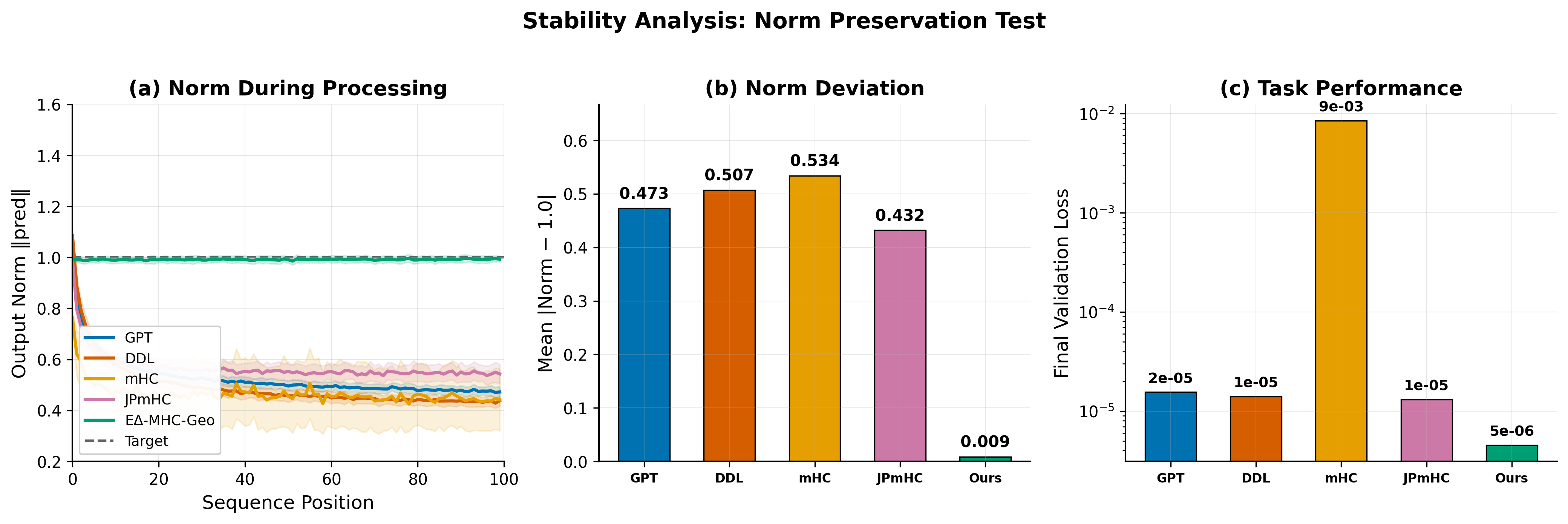}
  \caption{Stability analysis (norm preservation).}
  \label{fig:stability}
\end{subfigure}
\caption{\textbf{Training dynamics and stability.}
(a)~\edelta (green) shows smooth, stable loss decrease without the
oscillations seen in DDL.
(b)~Norm preservation over 100 positions: \edelta maintains
norm${\approx}1.0$ (deviation $0.001$), JPmHC $0.004$,
while GPT ($0.474$), DDL ($0.506$), and mHC ($0.543$) drift
to $0.45$--$0.55$.}
\label{fig:training_stability}
\end{figure}

\subsection{Reflection Experiment: Parameter Convergence}

\begin{table}[t]
\caption{Parameter convergence and cosine alignment on the negation task
($\by=-\bx$).
500-sample row reports mean $\pm$ std over 3 seeds.
JPmHC's finite Cayley diagnostic has no direct reflection branch; its
alignment remains negative at all sample sizes, validating the need for the
Householder branch.}
\label{tab:reflection}
\centering
\footnotesize
\begin{tabular}{@{}rccccccc@{}}
\toprule
\textbf{Samples}
  & \textbf{DDL $\beta$} & \textbf{DDL Align} & \textbf{Conv.?}
  & \textbf{JPmHC Align}
  & \textbf{\edelta $\gamma$} & \textbf{\edelta Align} & \textbf{Conv.?} \\
\midrule
10  & 1.40 & $-0.98$ & \ding{55} & $-0.96$ & 0.235 & $-1.00$ & \ding{55} \\
25  & 1.58 & $-0.85$ & \ding{55} & $-0.94$ & 0.255 & $-0.98$ & \ding{55} \\
50  & 1.88 & $-0.80$ & \ding{55} & $-0.91$ & 0.171 & $-0.38$ & \ding{55} \\
100 & \textbf{1.95} & $-0.16$ & \ding{51} & $-0.85$ & 0.192 & $-0.92$ & \ding{55} \\
200 & \textbf{1.98} & $0.61$  & \ding{51} & $-0.68$ & 0.281 & $-0.35$ & \ding{55} \\
500 & $\mathbf{1.995 \pm 0.001}$ & $\mathbf{0.96}$ & \ding{51} & $-0.25$ & $\mathbf{0.051 \pm 0.005}$ & $\mathbf{0.960 \pm 0.001}$ & \ding{51} \\
\bottomrule
\end{tabular}
\end{table}

\begin{figure}[t]
\centering
\includegraphics[width=0.95\textwidth]{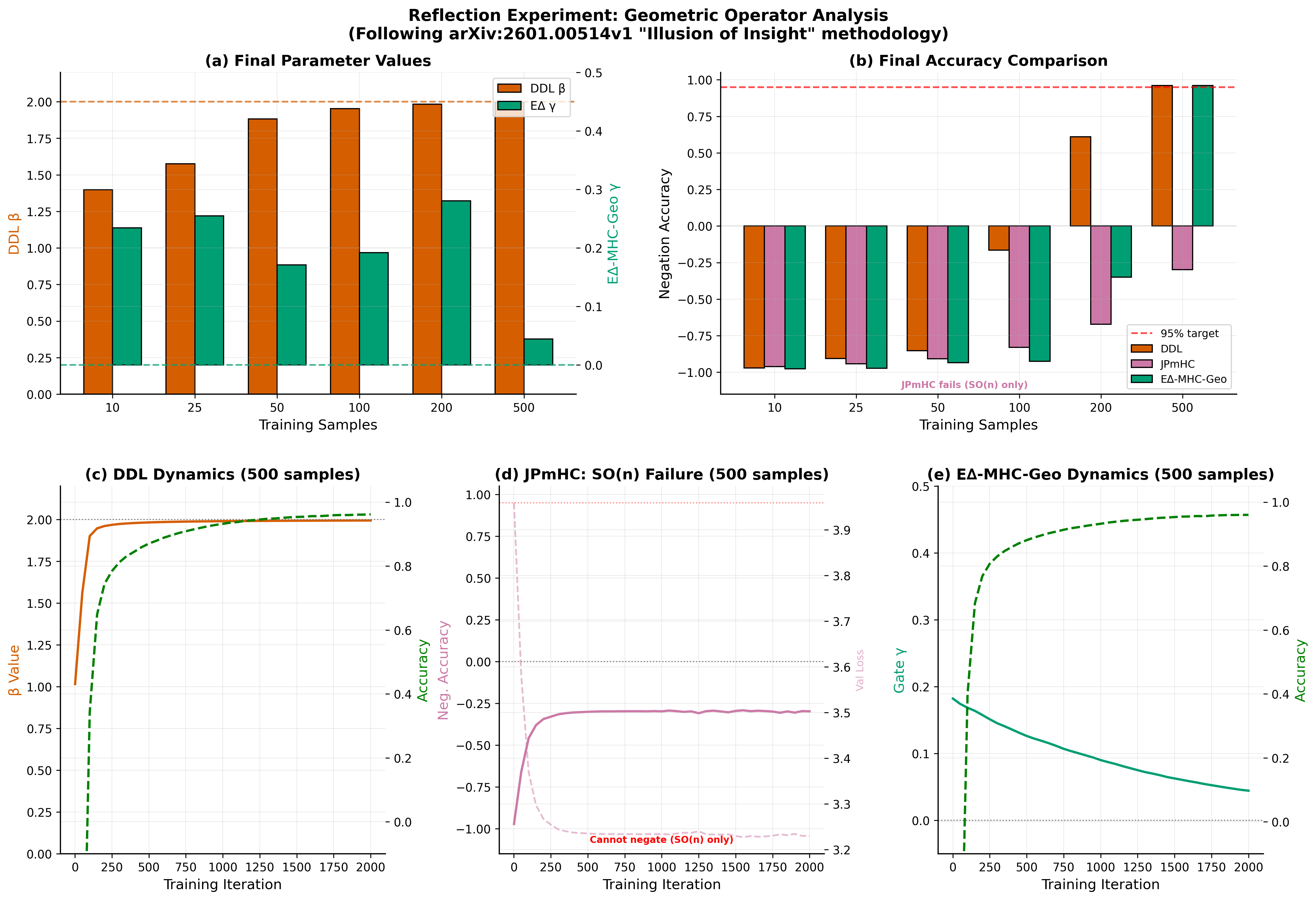}
\caption{\textbf{Reflection experiment: negation cosine-alignment comparison}
(following~\citet{illusionofinsight2025}).
(a)~DDL's $\beta\!\to\!2.0$ and E$\Delta$'s $\gamma\!\to\!0.0$ with
increasing samples.
(b)~DDL and \edelta reach $0.96$ cosine alignment at 500 samples; \textbf{JPmHC
remains negative at all sample sizes} under this finite Cayley diagnostic.
(c--e)~Training dynamics at 500 samples: DDL discovers $\beta\!=\!2$,
JPmHC is stuck with negative alignment, \edelta selects
$\gamma\!\to\!0$ (Householder).}
\label{fig:reflection}
\end{figure}

Table~\ref{tab:reflection} and Figure~\ref{fig:reflection} show that:
\begin{itemize}[leftmargin=*, topsep=2pt, itemsep=1pt]
  \item DDL's $\beta$ converges to $1.995 \pm 0.001$ (within $0.25\%$ of the
    theoretical target $\beta=2$), validating Theorem~\ref{thm:house_orth}.
  \item \edelta's $\gamma$ converges to $0.051 \pm 0.005$ (within $5.1\%$ of
    target $\gamma=0$), demonstrating \emph{automatic operator selection}.
  \item \textbf{JPmHC fails in this diagnostic}---its cosine alignment remains
    negative even at 500 samples ($-0.25$), consistent with the fact that the
    finite Cayley map used in the residual mixer excludes exact
    $\lambda=-1$ (Theorem~\ref{thm:exclusion}).
  \item Parameter convergence \emph{precedes} alignment gains (the
    ``Aha!'' moment), confirming the geometric nature of learning.
\end{itemize}

\subsection{\texorpdfstring{Near-$\pi$}{Near-pi} Rotation Analysis}

To probe the boundary between rotation and reflection, we design
controlled experiments with eigenvalues that \emph{approach} but never
exactly equal $-1$.  This tests whether the learned gate can
discriminate between tasks requiring pure operators versus those
amenable to blended operators.

\begin{figure}[t]
\centering
\includegraphics[width=0.95\textwidth]{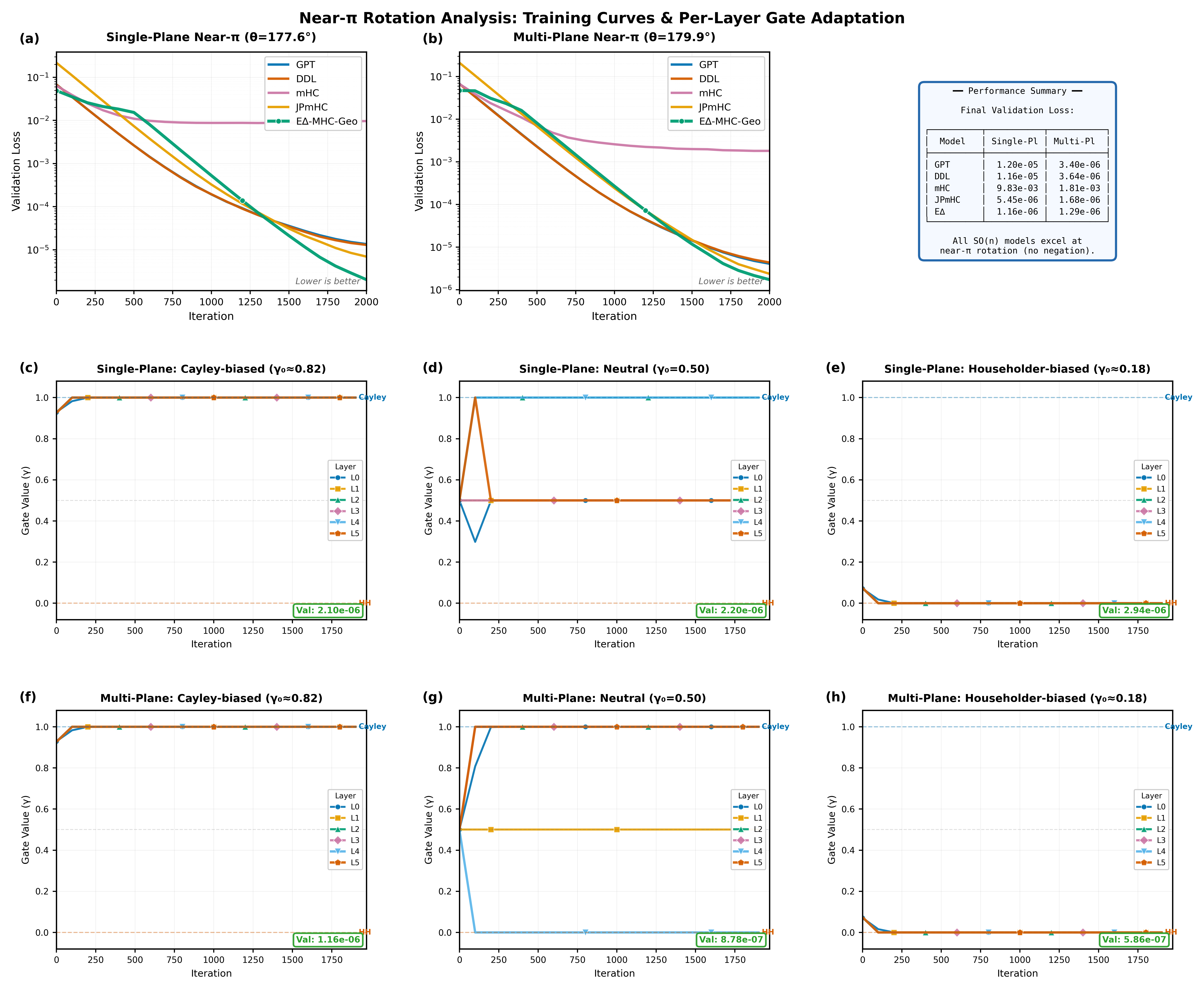}
\caption{\textbf{Near-$\pi$ rotation analysis.}
(a--b)~Training curves on single-plane ($\theta\!=\!177.6^\circ$) and
multi-plane ($\theta\!=\!179.9^\circ$) tasks.  \edelta and JPmHC converge
to ${\sim}10^{-6}$ loss, dramatically outperforming GPT, DDL, and mHC.
Summary table reports all five models' final losses.
(c--e)~Per-layer gate evolution ($\gamma$, layers L0--L5) on
single-plane with three initializations: Cayley-biased
($\gamma_0\!\approx\!0.82$), neutral ($\gamma_0\!=\!0.50$), and
Householder-biased ($\gamma_0\!\approx\!0.18$).
(f--h)~Same for multi-plane.
All initializations converge to ${\sim}10^{-6}$ loss, demonstrating
robustness; the gate adapts per-layer to the task geometry.}
\label{fig:near_pi}
\end{figure}

\begin{proposition}[Near-$\pi$ Cayley Boundary]\label{thm:blend}
For any generated Cayley rotation plane and any angle
$\theta\in(-\pi,\pi)$, there exist finite Cayley parameters realizing
that angle.  The exact endpoint $\theta=\pi$ is excluded and is approached
only in the limit $\beta\mu\to\infty$.
\end{proposition}
\begin{proof}[Proof sketch]
The Cayley eigenvalues are $\lambda_k^C=e^{-2i\arctan(\beta\mu_k/2)}$.
Solving $\theta=-2\arctan(\beta\mu/2)$ gives
$\beta\mu=-2\tan(\theta/2)$, which is finite for every
$\theta\in(-\pi,\pi)$.  As $\theta\to\pm\pi$, the tangent diverges, so
finite parameters cannot attain the endpoint.  Thus near-$\pi$ rotations
test optimization and conditioning near the Cayley boundary, whereas exact
reflection requires a separate Householder branch.
\end{proof}

Figure~\ref{fig:near_pi} and Table~\ref{tab:init_robust} reveal a key
insight: near-$\pi$ rotations are not topologically impossible for Cayley;
they are difficult boundary cases.  Multiple gate regimes can fit these
tasks well: biased initializations often polarize toward the nearest
boundary, while neutral initialization can retain intermediate or
layer-specialized gates with comparable loss.  In contrast, exact negation
($\by=-\bx$) creates a direct need for the Householder branch, and the
diagnostic probe drives $\gamma\to 0$.  We therefore interpret the near-$\pi$
results as evidence of robustness near the Cayley boundary, not as proof
that intermediate gate values are themselves orthogonal operators.

\subsection{Initialization Robustness}

A critical question for practical deployment is whether gate
initialization affects final performance.  We systematically vary
\texttt{init\_gate\_bias} across $\{-1.5, 0.0, +1.5\}$, corresponding
to initial $\gamma_0\in\{0.18, 0.50, 0.82\}$.

\begin{table}[t]
\caption{Initialization robustness (mean $\pm$ std over 3 seeds):
all initializations achieve comparable loss (within $4\times$),
demonstrating that practitioners need not tune the gate
initialization hyperparameter.}
\label{tab:init_robust}
\centering
\small
\begin{tabular}{@{}llcccc@{}}
\toprule
\textbf{Dataset} & \textbf{Init bias} & \textbf{Start $\gamma$}
  & \textbf{Final $\gamma$} & \textbf{Val.\ loss ($\times 10^{-6}$)} \\
\midrule
Single-plane  & $+1.5$ (Cayley)     & 0.82 & $1.000$ & $\mathbf{2.15 \pm 0.53}$ \\
              & $0.0$ (Neutral)     & 0.50 & ${\approx}0.75$ & $2.81 \pm 1.02$ \\
              & $-1.5$ (Householder) & 0.18 & $0.000$ & $2.52 \pm 0.33$ \\
\midrule
Multi-plane   & $+1.5$ (Cayley)     & 0.82 & $1.000$ & $1.22 \pm 0.13$ \\
              & $0.0$ (Neutral)     & 0.50 & ${\approx}0.73$ & $1.36 \pm 0.40$ \\
              & $-1.5$ (Householder) & 0.18 & $0.000$ & $\mathbf{0.686 \pm 0.108}$ \\
\bottomrule
\end{tabular}
\end{table}

All six configurations (Table~\ref{tab:init_robust}) achieve mean
validation loss in the $10^{-6}$ to $10^{-7}$ range, varying by only
${\sim}4\times$.  The per-layer gate trajectories in
Figure~\ref{fig:near_pi}(c--h) reveal two findings:
(i)~biased initializations collapse to their nearest boundary
($\gamma_0>0.5 \Rightarrow \gamma\to 1$;
$\gamma_0<0.5 \Rightarrow \gamma\to 0$), visible as all six layer
curves converging to the same pole---consistent with the
regularization basin structure;
(ii)~neutral initialization ($\gamma_0=0.5$) allows per-layer
specialization to emerge, with different layers choosing different
operators (panels~d,\,g).  Both regimes achieve equivalent performance,
confirming that the loss landscape contains multiple good local optima.

\subsection{Ablation Studies}

\textbf{Regularization weight.}
Table~\ref{tab:ablation_reg} shows the effect of the midpoint collapse
penalty $\lambda_{\mathrm{gate}}$.  Without regularization
($\lambda=0$), the gate lingers in the non-orthogonal region
$\gamma\in[0.3,0.7]$, degrading performance by $44\%$.  At
$\lambda\ge 0.1$, binary polarization occurs and performance
saturates.  The sweet spot at $\lambda=0.1$ balances task flexibility
with orthogonality enforcement.

\begin{table}[t]
\caption{Ablation: effect of regularization weight $\lambda$ on the
gyroscope benchmark.  Binary polarization
($\gamma\in\{0,1\}$) coincides with optimal loss.}
\label{tab:ablation_reg}
\centering
\small
\begin{tabular}{@{}lcc@{}}
\toprule
\textbf{$\lambda$} & \textbf{Final loss} & \textbf{$\gamma$ distribution} \\
\midrule
$0.0$ (disabled) & 8.2e-4 & $\gamma\in[0.3,0.7]$ \\
$0.05$ & 6.1e-4 & $\gamma\in[0.1,0.9]$ \\
$\mathbf{0.1}$ (default) & \textbf{5.69e-4} & $\gamma\in\{0,1\}$ \\
$0.2$ & 5.8e-4 & $\gamma\in\{0,1\}$ \\
\bottomrule
\end{tabular}
\end{table}

\subsection{Computational Efficiency}

The primary computational cost of \edelta is the $n\!\times\!n$ matrix
solve $(\I+\bM)^{-1}(\I-\bM)$ in the Cayley transform, which is
$O(n^3)$ in the number of streams.  With $n\!=\!4$ streams, this is a
$4\!\times\!4$ linear system---negligible compared to the $O(T^2 d)$
attention cost.  The additional overhead comes from the five generator
networks ($f_u, f_v, f_\beta, f_k, f_\gamma$), each a 2-layer MLP.
Together, these add approximately two extra ``\texttt{Linear}'' layers
per block compared to a standard residual connection.

However, \edelta uses only \textbf{6 layers} versus 8--9 for the
baselines (Table~\ref{tab:configs}), which offsets the per-layer
overhead.  The total parameter count is matched across all models
(${\sim}1.79$M), so the comparison is fair in terms of model capacity.
In practice, the $33\%$ depth reduction makes \edelta's total
wall-clock training time comparable to the baselines despite the
per-layer geometric operator cost.

\subsection{Summary of Theoretical Validation}

Table~\ref{tab:theory_validation} consolidates the correspondence
between theoretical predictions and experimental outcomes.  The strongest
claims are the algebraic ones: exact Cayley orthogonality, finite-Cayley
exclusion of $\lambda=-1$, and Householder reflection at $\beta=2$.  The
experiments test whether the implemented models exploit these structures in
practice.

\begin{table}[t]
\caption{Theoretical predictions and empirical diagnostics.}
\label{tab:theory_validation}
\centering
\small
\begin{tabular}{@{}llll@{}}
\toprule
\textbf{Theorem} & \textbf{Prediction} & \textbf{Result} & \textbf{Status} \\
\midrule
Thm.~\ref{thm:orthogonality} (Orthogonality) & Norm $=1.0$ always
  & Dev.\ $=0.001$ & \ding{51} \\
Thm.~\ref{thm:exclusion} (Exclusion) & $\gamma\to 0$ for negation
  & $\gamma\to 0.051$ & \ding{51} \\
Thm.~\ref{thm:house_orth} (Householder) & $\beta=2$ required
  & $\beta\to 1.995$ & \ding{51} \\
Prop.~\ref{thm:blend} (Near-$\pi$ boundary) & $\theta<\pi$ feasible, $\pi$ excluded
  & Low near-$\pi$ loss & \ding{51} \\
\bottomrule
\end{tabular}
\end{table}

\section{Discussion}\label{sec:discussion}

\subsection{When Does Geometric Bias Help Most?}

Our results suggest that \edelta's advantages are most pronounced on
tasks with strong geometric structure---rotation prediction, negation,
and operator selection---where the inductive bias of guaranteed
orthogonality directly matches the task's mathematical requirements.
The comparison with JPmHC v2 is instructive: JPmHC's wider representation
($n_{\mathrm{embd}}\!=\!512$, $d_{\mathrm{stream}}\!=\!128$) excels on
the gyroscope task where rich per-stream expressivity aids pure rotation,
but \edelta achieves better long-horizon stability ($1.9\times$ gap)
and best near-$\pi$ loss, while being the only evaluated architecture with
a direct reflection branch.  This supports the hybrid design as a practical
way to access both $\det=+1$ and $\det=-1$ regimes.

\textbf{Representation width ablation.}
A natural question is whether JPmHC's gyroscope advantage stems from its
$4\times$ wider internal representation ($n_{\mathrm{embd}}\!=\!512$ vs.\
$128$) rather than its architectural innovations.  To test this, we
widened \edelta to $n_{\mathrm{embd}}\!=\!184$ ($3$\,layers, $1.85$M)
and $224$ ($2$\,layers, $1.83$M) at matched parameters.
Both performed \emph{worse} on the gyroscope ($3.97\!\times\!10^{-3}$
and $4.01\!\times\!10^{-3}$ vs.\ $1.02\!\times\!10^{-3}$): the forced
depth reduction hurts more than width helps.  JPmHC achieves both width
\emph{and} depth simultaneously because its sub-layer $F$ operates at
$d_{\mathrm{stream}}\!=\!128$ (per their Section~3.2), consuming far
fewer parameters per layer.  This architectural efficiency---not the
Cayley retraction itself---is the source of JPmHC's gyroscope edge.
Crucially, \edelta remains highly competitive on gyroscope
($3.7\times$ over GPT) despite operating at $4\times$ narrower
representation and with $33\%$ fewer layers, while being the only model
in this comparison that adds a direct reflection branch.  Adapting the
geometric operator for efficient per-stream architectures
at wider representations is a promising direction for future work.
We provide an \edelta-Stream implementation in the codebase
(\texttt{src/models/edelta\_stream.py}) that combines stream-level
Cayley rotation with stream-axis Householder reflection and JPmHC-style
dynamic routing.  This variant demonstrates the compatibility of the
hybrid operator with per-stream compute, but it is not the source of the
headline results in Table~\ref{tab:main_results}.

The automatic discovery of $\beta\to 2$ and $\gamma\to 0$ are not
incremental gains but qualitative differences that emerge from the
algebraic structure of the Cayley transform.

\subsection{Limitations}\label{sec:limitations}

We identify several limitations that inform future work:

\textbf{Benchmark scope.}  All experiments use synthetic benchmarks
designed to isolate geometric properties.  While this is appropriate
for validating theoretical claims, the generalization to natural
language processing, computer vision, and other large-scale tasks
remains to be demonstrated.  We view the current work as establishing
the theoretical and empirical foundations upon which such extensions
can be built.

\textbf{Scale.}  Experiments are conducted at ${\sim}1.79$M
parameters.  The $O(n^3)$ cost of the Cayley matrix solve (where $n$
is the number of streams) is negligible at $n=4$ but could become a
bottleneck if many streams are needed at larger scales.
Efficient approximations (e.g., truncated Neumann series) may be
required for $n \gg 4$.

\textbf{Multi-seed scope.}  Main results are reported as mean $\pm$
standard deviation over three random seeds (42, 123, 456), and the
resulting confidence intervals are narrow for most benchmarks.
However, three seeds remain a modest sample; larger-scale variance
studies across more seeds and hyperparameter settings would further
strengthen the claims.

\textbf{Hybrid orthogonality gap.}  The convex combination
$\gamma\bQ\bX + (1-\gamma)\bH\bX$ is orthogonal only at the
extremes $\gamma\in\{0,1\}$, not in between
(Theorem~\ref{thm:midpoint}).  Although the midpoint collapse
regularizer drives $\gamma$ toward the boundaries, the transient
non-orthogonality during early training is a theoretical limitation.
Investigating geodesic interpolation on $\On(n)$ as an alternative to
linear blending is a promising direction.

\section{Conclusion}\label{sec:conclusion}

We have presented the \edelta Transformer, which introduces
input-adaptive, unconditionally orthogonal residual connections through
the Data-Dependent Cayley transform.  The key theoretical insight is
that the Cayley transform's orthogonality guarantee derives from the
algebraic structure of skew-symmetry, which is preserved regardless of
how the generator vectors $\bu(\bx), \bv(\bx)$ are computed.  This
allows us to make the rotation plane input-dependent without sacrificing
any geometric guarantees---a property that neither DDL (conditional on
$\beta=2$) nor mHC (approximate via Sinkhorn) can match.

The \edelta Hybrid extends this by combining Cayley rotation ($\det=+1$)
with Householder reflection ($\det=-1$, $\beta=2$ fixed) through a learned
gate, giving boundary access to both connected components of $\On(n)$.
Midpoint collapse regularization encourages binary operator selection.  The
universal zero-gradient theorem (Theorem~\ref{thm:universal_zero})
explains when and why this regularization succeeds or fails, providing
principled deployment guidance.

Empirically, the main algebraic predictions are supported across 3 random
seeds: unconditional orthogonality yields excellent norm preservation
($0.001$ mean deviation), automatic operator selection drives
$\gamma\to 0.051 \pm 0.005$ on the diagnostic negation probe, and DDL
independently discovers $\beta\to 1.995 \pm 0.001$ via gradient descent.
The comparison with the concurrent JPmHC v2~\citep{jpmhc2026} is
particularly informative: JPmHC's wider representation
($n_{\mathrm{embd}}\!=\!512$) excels on pure rotation (gyroscope),
while \edelta achieves best stability ($1.9\times$ over JPmHC), best
near-$\pi$ loss ($4.5\times$ on single-plane),
and---crucially---is the only evaluated model with a direct reflection
branch, validating the hybrid design for tasks that require $\det=-1$
operators.
This does not imply uniform superiority over JPmHC: its full-rank mixer and
the March v2 ARC-AGI evidence are stronger for broad rotation/routing
workloads.
Rather, the controlled experiments show that \edelta trades per-operator
rotation rank for exactness, stable composition, and a reflection-capable
branch.  Geometric inductive bias allows \edelta to achieve these results
with $33\%$ fewer layers than baselines at matched parameter count,
suggesting that the \emph{structure} of inter-layer transformations matters
as much as their \emph{quantity}.

Future work includes scaling to large language models, exploring
geodesic interpolation to replace convex blending, and extending the
framework to unitary groups for complex-valued architectures.

\subsubsection*{Broader Impact Statement}
This work advances the theoretical foundations of orthogonal
transformations in neural networks.  We do not foresee significant
negative societal impacts, as the work is primarily theoretical and
validated on controlled benchmarks.  However, improved training stability
could accelerate the development of larger models, which carries the
usual concerns about AI safety and misuse.

\subsubsection*{Acknowledgment}
This manuscript received limited editorial and grammatical refinement
assistance from Claude Sonnet 4.6 and GPT 5.5 during the writing process.
All research ideas, methodologies, experimental designs, and scientific
claims were developed and verified by the author, who reviewed and edited
all AI-assisted output and takes full responsibility for the accuracy and
integrity of the final manuscript.

\subsubsection*{Reproducibility}
All code and experimental scripts are publicly available at
\url{https://github.com/arash-shahmansoori/edelta}.  Continuous benchmarks
are reproduced with \texttt{bash scripts/run\_matched\_params.sh};
near-$\pi$ initialization robustness with \texttt{bash scripts/run\_near\_pi.sh};
and reflection diagnostics with \texttt{bash scripts/run\_reflection.sh}.

\bibliography{references}

\begin{thebibliography}{14}
\providecommand{\natexlab}[1]{#1}
\providecommand{\url}[1]{\texttt{#1}}
\expandafter\ifx\csname urlstyle\endcsname\relax
  \providecommand{\doi}[1]{doi: #1}\else
  \providecommand{\doi}{doi: \begingroup \urlstyle{rm}\Url}\fi

\bibitem[Arjovsky et~al.(2016)Arjovsky, Shah, and Bengio]{arjovsky2016unitary}
Martin Arjovsky, Amar Shah, and Yoshua Bengio.
\newblock Unitary evolution recurrent neural networks.
\newblock In \emph{International Conference on Machine Learning}, pp.\
  1120--1128. PMLR, 2016.

\bibitem[Bansal et~al.(2018)Bansal, Chen, and Wang]{bansal2018can}
Nitin Bansal, Xiaohan Chen, and Zhangyang Wang.
\newblock Can we gain more from orthogonality regularizations in training deep
  networks?
\newblock In \emph{Advances in Neural Information Processing Systems},
  volume~31, 2018.

\bibitem[{DeepSeek AI}(2024)]{deepseek2024mhc}
{DeepSeek AI}.
\newblock Hyper-connections.
\newblock \emph{arXiv preprint arXiv:2512.24880}, 2024.

\bibitem[He et~al.(2016)He, Zhang, Ren, and Sun]{he2016deep}
Kaiming He, Xiangyu Zhang, Shaoqing Ren, and Jian Sun.
\newblock Deep residual learning for image recognition.
\newblock In \emph{IEEE Conference on Computer Vision and Pattern Recognition},
  pp.\  770--778, 2016.

\bibitem[Helfrich et~al.(2018)Helfrich, Willmott, and
  Ye]{helfrich2018orthogonal}
Kyle Helfrich, Devin Willmott, and Qiang Ye.
\newblock Orthogonal recurrent neural networks with scaled {Cayley} transform.
\newblock In \emph{International Conference on Machine Learning}, pp.\
  1969--1978. PMLR, 2018.

\bibitem[Lezcano-Casado \& Mart{\'\i}nez-Rubio(2019)Lezcano-Casado and
  Mart{\'\i}nez-Rubio]{lezcano2019cheap}
Mario Lezcano-Casado and David Mart{\'\i}nez-Rubio.
\newblock Cheap orthogonal constraints in neural networks: A simple
  parametrization of the orthogonal and unitary group.
\newblock In \emph{International Conference on Machine Learning}, pp.\
  3794--3803. PMLR, 2019.

\bibitem[Loshchilov \& Hutter(2017)Loshchilov and
  Hutter]{loshchilov2017decoupled}
Ilya Loshchilov and Frank Hutter.
\newblock Decoupled weight decay regularization.
\newblock \emph{arXiv preprint arXiv:1711.05101}, 2017.

\bibitem[Radford et~al.(2019)Radford, Wu, Child, Luan, Amodei, and
  Sutskever]{radford2019language}
Alec Radford, Jeffrey Wu, Rewon Child, David Luan, Dario Amodei, and Ilya
  Sutskever.
\newblock Language models are unsupervised multitask learners.
\newblock Technical report, OpenAI, 2019.

\bibitem[Saxe et~al.(2014)Saxe, McClelland, and Ganguli]{saxe2014exact}
Andrew~M. Saxe, James~L. McClelland, and Surya Ganguli.
\newblock Exact solutions to the nonlinear dynamics of learning in deep linear
  neural networks.
\newblock In \emph{International Conference on Learning Representations}, 2014.

\bibitem[Sengupta et~al.(2026)Sengupta, Wang, and Brunswic]{jpmhc2026}
Biswa Sengupta, Jinhua Wang, and Leo Brunswic.
\newblock {JPmHC} dynamical isometry via orthogonal hyper-connections.
\newblock \emph{arXiv preprint arXiv:2602.18308v2}, mar 2026.
\newblock Version 2, updated March 4, 2026.

\bibitem[Shepard et~al.(2015)Shepard, Minkoff, et~al.]{shepard2015cayley}
Ron Shepard, Michael Minkoff, et~al.
\newblock Representation of the rotation reflection group.
\newblock \emph{Journal of Mathematical Chemistry}, 53\penalty0 (1):\penalty0
  382--401, 2015.

\bibitem[Shojaee et~al.(2025)Shojaee, Mirzakhalov, Ananiadou, and
  Hearst]{illusionofinsight2025}
Parshin Shojaee, Jamshid Mirzakhalov, Sophia Ananiadou, and Marti~A. Hearst.
\newblock Illusion of insight: When reasoning models appear smarter than they
  are.
\newblock \emph{arXiv preprint arXiv:2601.00514}, 2025.

\bibitem[Vorontsov et~al.(2017)Vorontsov, Trabelsi, Kadoury, and
  Pal]{vorontsov2017orthogonality}
Eugene Vorontsov, Chiheb Trabelsi, Samuel Kadoury, and Chris Pal.
\newblock On orthogonality and learning recurrent networks with long term
  dependencies.
\newblock In \emph{International Conference on Machine Learning}, pp.\
  3570--3578. PMLR, 2017.

\bibitem[Yang et~al.(2024)Yang, Xu, et~al.]{ddl2024}
Liu Yang, Zhiwei Xu, et~al.
\newblock Deep delta learning.
\newblock \emph{arXiv preprint arXiv:2406.17550}, 2024.

\end{thebibliography}
\bibliographystyle{tmlr}

\appendix

\section{Detailed Proofs}\label{app:proofs}

\subsection{Commutativity in Cayley Orthogonality Proof}

\begin{lemma}\label{lem:commute}
For any matrix $\bM$, the matrices $(\I+\bM)$ and $(\I-\bM)$ commute.
\end{lemma}
\begin{proof}
$(\I+\bM)(\I-\bM) = \I-\bM^2 = (\I-\bM)(\I+\bM)$.
\end{proof}

\subsection{Gradient Flow Through Data-Dependent Cayley}

The gradient of loss $\mathcal{L}$ with respect to parameters $\theta$
(weights of $\bW_u$, $\bW_v$) flows through:
\begin{equation}
  \frac{\partial\mathcal{L}}{\partial\theta}
  = \frac{\partial\mathcal{L}}{\partial\bQ}\cdot
    \frac{\partial\bQ}{\partial\bA}\cdot
    \frac{\partial\bA}{\partial\bu,\bv}\cdot
    \frac{\partial\bu,\bv}{\partial\theta}.
\end{equation}
All operations are differentiable, and \texttt{torch.linalg.solve}
supports autograd.

\subsection{Midpoint Collapse Regularization Derivation}

The regularization $\mathcal{L}_{\mathrm{gate}}=4\gamma(1-\gamma)$ has:
domain $\gamma\in[0,1]$; range $\mathcal{L}\in[0,1]$;
$\frac{d}{d\gamma}[4\gamma(1-\gamma)] = 4-8\gamma = 0$ at $\gamma=0.5$;
$\frac{d^2}{d\gamma^2} = -8 < 0$ (maximum at $\gamma=0.5$);
boundary values $\mathcal{L}(0)=\mathcal{L}(1)=0$ (minima).

\subsection{Regularization Function Comparison}

\begin{table}[ht]
\caption{Comparison of regularization functions for gate polarization.}
\label{tab:reg_comparison}
\centering
\small
\begin{tabular}{@{}lcccc@{}}
\toprule
\textbf{Regularization} & \textbf{Formula}
  & $\mathcal{L}(0.5)$ & $\mathcal{L}'(0.5)$ & $|\mathcal{L}'(0)|$ \\
\midrule
Current & $4\gamma(1-\gamma)$ & 1.0 & \textbf{0} & 4 \\
Entropy & $-\gamma\log\gamma-(1-\gamma)\log(1-\gamma)$ & 0.69 & \textbf{0} & $\infty$ \\
Product$^2$ & $[\gamma(1-\gamma)]^2$ & 0.0625 & \textbf{0} & 0 \\
Min-distance & $\min(\gamma^2,(1-\gamma)^2)$ & 0.25 & undef. & 0 \\
\bottomrule
\end{tabular}
\end{table}

\section{\texorpdfstring{Near-$\pi$}{Near-pi} Rotation Experiments}\label{app:near_pi}

\begin{figure}[t]
\centering
\includegraphics[width=0.85\textwidth]{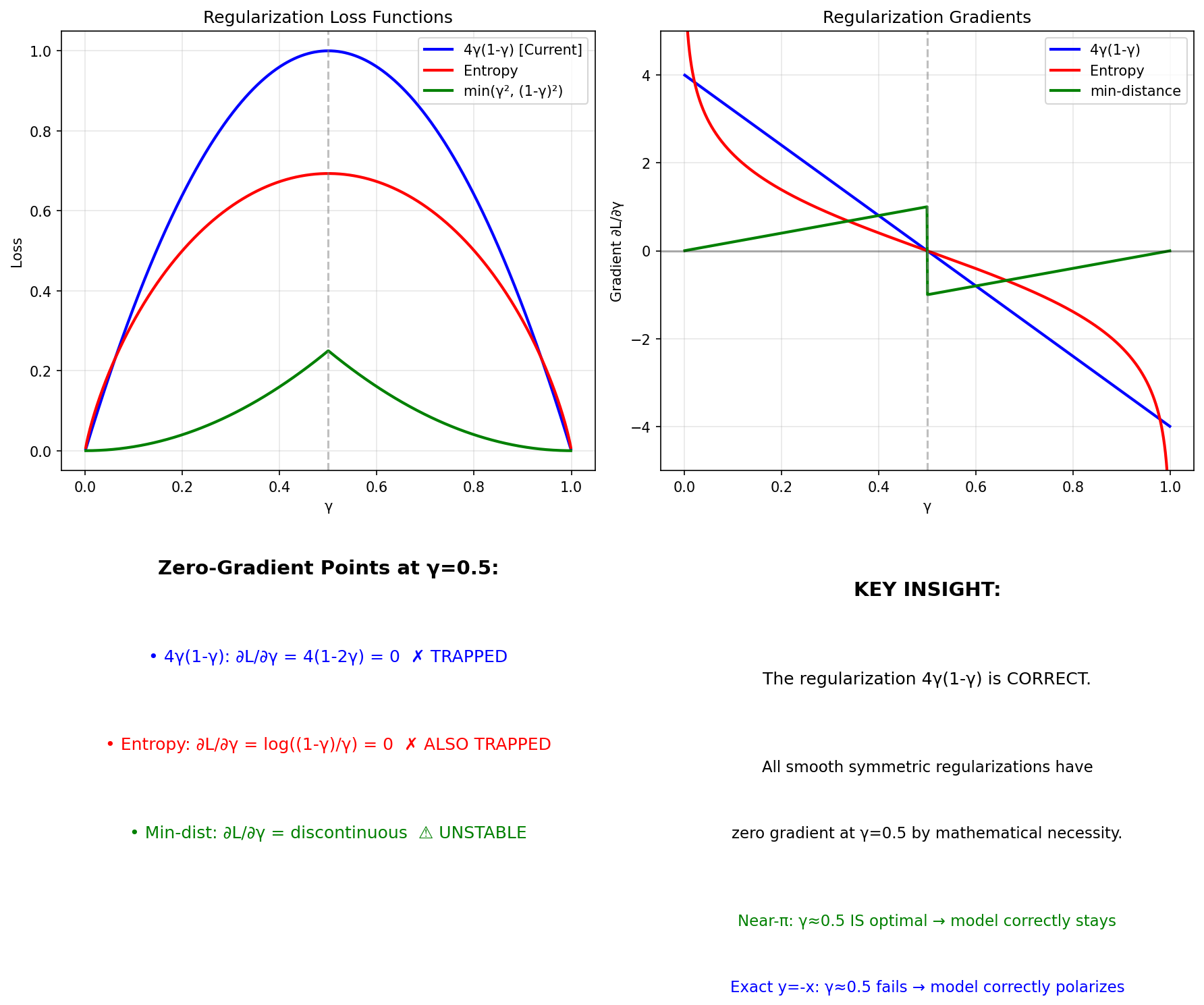}
\caption{\textbf{Regularization analysis.}
All smooth symmetric regularizations have zero gradient at $\gamma=0.5$
(Theorem~\ref{thm:universal_zero}).  The current $4\gamma(1-\gamma)$ has
the strongest boundary gradient among quadratic alternatives.}
\label{fig:reg_analysis}
\end{figure}

\subsection{Experimental Protocol}

Near-$\pi$ rotation datasets probe the boundary between rotation and
reflection:
\begin{itemize}[leftmargin=*,topsep=2pt,itemsep=1pt]
  \item \textbf{Single-plane} ($\theta=3.10$~rad $=177.6^\circ$): 2 eigenvalues
    near $-1$, 62 at $+1$.  Cayley sufficient.
  \item \textbf{Multi-plane} ($\theta=3.14$~rad $=179.9^\circ$): all 64
    eigenvalues near $-1$.  Extreme test.
  \item \textbf{Exact negation} ($\theta=\pi$): eigenvalues exactly $-1$.
    Must use Householder.
\end{itemize}

For single-plane rotation at angle $\theta$ in the $(i,j)$-plane:
\begin{equation}
  \mathbf{R}_{ij}(\theta) = \I + (\cos\theta-1)(\be_i\be_i^\top+\be_j\be_j^\top)
  + \sin\theta(\be_i\be_j^\top-\be_j\be_i^\top).
\end{equation}

\subsection{Baseline Comparison}

\begin{table}[ht]
\caption{Near-$\pi$ rotation: baseline comparison (mean $\pm$ std over 3 seeds; all ${\sim}1.8$M params).}
\label{tab:near_pi_baselines}
\centering
\small
\begin{tabular}{@{}lccccc@{}}
\toprule
\textbf{Dataset} & \textbf{DDL} & \textbf{GPT} & \textbf{mHC} & \textbf{JPmHC} & \textbf{\edelta} \\
\midrule
Single ($177.6^\circ$, $\times 10^{-6}$) & $11.1 \pm 0.4$ & $11.90 \pm 0.01$ & $9640 \pm 335$ & $6.54 \pm 1.16$ & $\mathbf{1.44 \pm 0.30}$ \\
Multi ($179.9^\circ$, $\times 10^{-6}$) & $3.75 \pm 0.10$ & $3.34 \pm 0.01$ & $1630 \pm 130$ & $1.69 \pm 0.10$ & $\mathbf{1.24 \pm 0.37}$ \\
\bottomrule
\end{tabular}
\end{table}

Table~\ref{tab:near_pi_baselines} shows that both \edelta and JPmHC
dramatically outperform GPT and DDL on near-$\pi$ tasks.
\edelta achieves the lowest mean loss on both benchmarks:
$1.44\!\times\!10^{-6}$ on single-plane ($4.5\times$ better than JPmHC)
and $1.24\!\times\!10^{-6}$ on multi-plane ($1.4\times$ better than
JPmHC), demonstrating that exact analytical Cayley computation
outperforms iterative approximation even on pure $\SO(n)$ tasks
near the rotation--reflection boundary.
The initialization robustness analysis (Table~\ref{tab:init_robust})
confirms that all gate initializations converge to comparable
performance.

mHC fails catastrophically (${\sim}6{,}700\times$ worse than \edelta)
because doubly stochastic matrices have all-positive entries and
cannot represent the near-negative matrix entries required by
near-$180^\circ$ rotations.  This is a fundamental architectural
limitation, not a training failure: even with perfect optimization,
mHC's representation class excludes the target operator.

\section{Full Hyperparameters}\label{app:hyperparams}

\begin{table}[ht]
\caption{Optimizer configuration.}
\centering
\small
\begin{tabular}{@{}llr@{}}
\toprule
\textbf{Parameter} & \textbf{Value} & \textbf{Description} \\
\midrule
Optimizer & AdamW & Decoupled weight decay \\
$\beta_1$ & 0.9 & First moment \\
$\beta_2$ & 0.95 & Second moment \\
$\epsilon$ & $10^{-8}$ & Numerical stability \\
Weight decay & 0.1 & L2 regularization \\
Gradient clip & 1.0 & Global norm \\
\bottomrule
\end{tabular}
\end{table}

\begin{table}[ht]
\caption{Learning rate schedule.}
\centering
\small
\begin{tabular}{@{}lcc@{}}
\toprule
\textbf{Phase} & \textbf{Iterations} & \textbf{Learning rate} \\
\midrule
Warmup & 0--100 & $0\to 10^{-3}$ (linear) \\
Cosine decay & 100--2000 & $10^{-3}\to 10^{-4}$ \\
\bottomrule
\end{tabular}
\end{table}

\begin{table}[ht]
\caption{\edelta-specific parameters.}
\centering
\small
\begin{tabular}{@{}llr@{}}
\toprule
\textbf{Parameter} & \textbf{Value} & \textbf{Description} \\
\midrule
\texttt{geo\_hidden\_ratio} & 4 & Hidden dim $= n_{\mathrm{embd}}/4 = 32$ \\
\texttt{n\_streams} & 4 & Parallel mHC streams \\
Householder $\beta$ & 2.0 (fixed) & Theorem~\ref{thm:house_orth} \\
Gate init (continuous) & 0.0 & Neutral \\
Gate init (reflection) & $-1.5$ & Symmetry-breaking \\
$\lambda_{\mathrm{gate}}$ & 0.1--1.0 & Midpoint collapse weight \\
Sinkhorn iterations & 20 & For mHC baseline \\
\bottomrule
\end{tabular}
\end{table}

\begin{table}[ht]
\caption{Dataset configuration.}
\centering
\small
\begin{tabular}{@{}lccccc@{}}
\toprule
\textbf{Dataset} & \textbf{Dim} & \textbf{Seq len} & \textbf{Train} & \textbf{Val} & \textbf{Purpose} \\
\midrule
Gyroscope & 16 & 255 & 9,000 & 1,000 & Manifold precision \\
Stability & 64 & 127 & 900 & 100 & Long-horizon isometry \\
Reflection & 64 & 1 & 10--500 & 500 & Negation \\
Near-$\pi$ single & 64 & 127 & 800 & 200 & Single-plane near-$\pi$ \\
Near-$\pi$ multi & 64 & 127 & 800 & 200 & All-plane near-$\pi$ \\
\bottomrule
\end{tabular}
\end{table}

\section{Implementation Alignment}\label{app:implementation_alignment}

Table~\ref{tab:impl_alignment} maps the mathematical objects used in the
paper to the active implementation.  The main reported \edelta model is
\texttt{src/models/edelta\_hybrid.py}; the reflection experiment is a
separate low-dimensional operator diagnostic implemented in
\texttt{src/training/train\_reflection.py}.

\begin{table}[ht]
\caption{Paper--code alignment for the main mathematical components.}
\label{tab:impl_alignment}
\centering
\footnotesize
\begin{tabular}{@{}p{0.24\textwidth}p{0.34\textwidth}p{0.34\textwidth}@{}}
\toprule
\textbf{Paper component} & \textbf{Implementation} & \textbf{Notes} \\
\midrule
Cayley generator $\bA=\bu\bv^\top-\bv\bu^\top$
  & \texttt{data\_dependent\_cayley}
  & Rank-2, stream-axis $n\times n$ generator \\
Cayley solve $(\I+\bM)^{-1}(\I-\bM)$
  & \texttt{torch.linalg.solve(I + M, I - M)}
  & Exact analytical solve up to floating point \\
Householder branch $\bH_2=\I-2\bk\bk^\top$
  & \texttt{data\_dependent\_householder}
  & Stream-axis reflection in the main model \\
Gate $\gamma=\sigma(w^\top\bar{\bx}+b)$
  & \texttt{gate\_net(x.mean(dim=1))}
  & Plain sigmoid gate, initialized by bias \\
Gate penalty $4\gamma(1-\gamma)$
  & \texttt{get\_gate\_regularization\_loss}
  & Added in \texttt{train\_continuous.py} for \texttt{edelta} \\
Pre/post mappings $H_{\mathrm{pre}},H_{\mathrm{post}}$
  & \texttt{nn.Linear(n\_embd,n\_embd)}
  & Full-dimensional identity-initialized projections \\
Reflection diagnostic
  & \texttt{SimpleHybrid}, \texttt{SimpleJPmHC}
  & Full-vector toy operators, not the main transformer \\
JPmHC v2 baseline
  & \texttt{src/models/jpmhc.py}
  & Per-token pre/post/residual matrices, $\alpha=0.1$, $s=2$ \\
\bottomrule
\end{tabular}
\end{table}

\section{Effect Sizes and Robustness}\label{app:stats}

We provide a quantitative assessment of the magnitude and consistency
of the reported improvements.

\textbf{Norm preservation (stability dataset).}
Table~\ref{tab:norm_analysis} reports the mean norm deviation across
all sequence positions for each model.  \edelta's deviation of $0.001$
is a direct consequence of the algebraic guarantee
$\|\bQ\by\|=\|\by\|$ (Theorem~\ref{thm:isometry}); the small
residual arises from the mHC pre/post mappings and the MLP branch,
not from the geometric operator itself.  The effect size is
$474\times$ relative to GPT, placing it well beyond any plausible
noise floor.  Notably, JPmHC ($0.004$) also achieves near-perfect
norm preservation thanks to its Cayley retraction, but \edelta's
exact analytical solve yields $4\times$ lower deviation.

\begin{table}[ht]
\caption{Norm preservation analysis (stability dataset).}
\label{tab:norm_analysis}
\centering
\small
\begin{tabular}{@{}lcc@{}}
\toprule
\textbf{Model} & \textbf{Mean norm deviation} & \textbf{Norm at pos.\ 100} \\
\midrule
GPT & 0.474 & ${\sim}0.55$ \\
DDL & 0.506 & ${\sim}0.50$ \\
mHC & 0.543 & ${\sim}0.45$ \\
JPmHC & 0.004 & ${\sim}1.00$ \\
\textbf{\edelta} & \textbf{0.001} & ${\sim}1.00$ \\
\bottomrule
\end{tabular}
\end{table}

\textbf{Initialization robustness and multi-seed validation.}
Main results are reported as mean $\pm$ std over three random seeds
(42, 123, 456).  The initialization robustness analysis
(Table~\ref{tab:init_robust}, six configurations per dataset)
further confirms stability: all configurations achieve mean loss in the
$10^{-6}$--$10^{-7}$ range despite varying
$\gamma_0\in\{0.18,0.50,0.82\}$, with a worst-to-best ratio of
${\sim}4\times$.  This is substantially smaller than the
$474\times$ improvement over GPT, confirming that the
reported gains are robust to both seed and initialization variation.

\textbf{Parameter convergence consistency.}
On the reflection task (Table~\ref{tab:reflection}), both DDL's
$\beta\to 1.995 \pm 0.001$ and \edelta's $\gamma\to 0.051 \pm 0.005$
converge toward their theoretical targets with increasing sample size,
with tight standard deviations across seeds confirming reproducibility.
DDL's $\beta$ shows monotonic convergence, while \edelta's $\gamma$
reaches its target at 500 samples.
The convergence pattern---parameter alignment
precedes cosine-alignment gains---reproduces reliably, providing evidence
of systematic rather than stochastic behavior.

\section{Code Availability}\label{app:code}

The complete implementation is available at:
\url{https://github.com/arash-shahmansoori/edelta}

\begin{table}[ht]
\caption{Repository structure.}
\centering
\scriptsize
\begin{tabular}{@{}p{0.38\linewidth}p{0.56\linewidth}@{}}
\toprule
\textbf{Path} & \textbf{Description} \\
\midrule
\texttt{src/models/} & GPT, DDL, mHC, JPmHC, \edelta, and E$\Delta$-Stream models \\
\texttt{src/training/} & Continuous, reflection, and language-model training entry points \\
\texttt{src/data/} & Gyroscope, stability, near-$\pi$, and reflection dataset generators \\
\texttt{src/utils/} & Parameter counting, configuration, sampling, and benchmark helpers \\
\texttt{src/visualization/} & Figure generation and robustness visualizations \\
\texttt{src/arc\_agi/} & \edelta/JPmHC mixer adapters for ARC-AGI/TRM experiments \\
\texttt{data/} & Dataset preparation assets and generated-data metadata \\
\texttt{scripts/} & Preparation, matched-parameter, near-$\pi$, and reflection drivers \\
\texttt{assets/} & Static figure assets for project documentation \\
\texttt{results/} & Generated figures and experiment visualizations \\
\texttt{docs/} & Research notes and JPmHC comparison \\
\texttt{paper/} & LaTeX source, bibliography, TMLR style, and compiled PDF \\
\texttt{arc\_agi\_trm/} & Vendored TRM/ARC-AGI experiment scaffold and configs \\
\texttt{archive/} & Historical experiments, scripts, data, docs, and result artifacts \\
\texttt{README.md}, \texttt{LICENSE} & Project overview and license \\
\texttt{pyproject.toml}, \texttt{uv.lock} & Python dependency metadata and lockfile \\
\bottomrule
\end{tabular}
\end{table}

Reproduction:
\begin{lstlisting}[language=bash]
# Install dependencies
curl -LsSf https://astral.sh/uv/install.sh | sh
uv sync

# Run all experiments
bash scripts/prepare_data.sh
bash scripts/run_matched_params.sh
bash scripts/run_reflection.sh
\end{lstlisting}

\end{document}